\documentclass{article}

\usepackage[final,nonatbib]{neurips_2019}

\usepackage[utf8]{inputenc} %
\usepackage[T1]{fontenc}    %
\usepackage{hyperref}       %
\usepackage{url}            %
\usepackage{booktabs}       %
\usepackage{amsfonts}       %
\usepackage{nicefrac}       %
\usepackage{microtype}      %
\usepackage{amsmath}
\usepackage{amsthm}
\usepackage{color}
\usepackage{enumitem}
\usepackage{graphicx}
\usepackage{subcaption}

\def\eg{\emph{e.g.}}
\def\ie{\emph{i.e.}}
\def\Hcal{{\mathcal H}}

\def\kmone{{k\text{--}1}}

\def\nmone{{n\text{--}1}}

\def\fm{{a}}
\def\xi{{x}}
\def\chi{{y}}

\def\R{{\mathbb R}}

\def\Z{{\mathbb Z}}
\def\Sbb{{\mathbb S}}
\def\Hc{{\mathcal H}}

\DeclareMathOperator{\Tr}{Tr}

\DeclareMathOperator{\E}{\mathbb{E}}

\newtheorem{theorem}{Theorem}
\newtheorem{proposition}[theorem]{Proposition}

\newtheorem{lemma}[theorem]{Lemma}
\newtheorem{corollary}[theorem]{Corollary}

\title{On the Inductive Bias of Neural Tangent Kernels}

\author{%
  Alberto Bietti \\
  Inria\thanks{Univ. Grenoble Alpes, Inria, CNRS, Grenoble INP, LJK, 38000 Grenoble, France} \\
  \texttt{alberto.bietti@inria.fr}
  \And
  Julien Mairal \\
  Inria\footnotemark[1] \\
  \texttt{julien.mairal@inria.fr}
}

\begin{document}

\maketitle

\begin{abstract}

State-of-the-art neural networks are heavily over-parameterized,
making the optimization algorithm a crucial ingredient for learning predictive models with
good generalization properties.
A recent line of work has shown that in a certain over-parameterized regime,
the learning dynamics of gradient descent are governed by a certain kernel obtained at initialization,
called the \emph{neural tangent kernel}.
We study the inductive bias of learning in such a regime by analyzing this kernel
and the corresponding function space (RKHS).
In particular, we study smoothness, approximation, and stability properties of functions with
finite norm, including stability to image deformations in the case of convolutional networks,
and compare to other known kernels for similar architectures.

\end{abstract}

\section{Introduction}

The large number of parameters in state-of-the-art deep neural networks
makes them very expressive, with the ability to approximate large classes of functions~\cite{hornik1989multilayer,pinkus1999approximation}.
Since many networks can potentially fit a given dataset, the optimization method, typically a variant
of gradient descent, plays a crucial role in selecting a model that generalizes well~\cite{neyshabur2015search}.

A recent line of work~\cite{allen2019convergence,chizat2018note,du2019bgradient,du2019agradient,jacot2018neural,li2018learning,zou2019stochastic} has shown that when training deep networks in a certain over-parameterized regime,
the dynamics of gradient descent behave like those of a linear model on (non-linear) features determined at initialization.
In the over-parameterization limit, these features correspond to a kernel known as
the \emph{neural tangent kernel}.
In particular, in the case of a regression loss, the obtained model behaves similarly to a
minimum norm kernel least squares solution, suggesting that this kernel may play a key role
in determining the inductive bias of the learning procedure and its generalization properties.
While it is still not clear if this regime is at play in state-of-the-art deep networks, there is
some evidence that this phenomenon of ``lazy training''~\cite{chizat2018note}, where weights only move
very slightly during training, may be relevant for early stages of training and for the outmost layers
of deep networks~\cite{lee2019wide,zhang2019all}, motivating a better understanding of its properties.

In this paper, we study the inductive bias of this regime by studying properties of functions in
the space associated with the neural tangent kernel for a given architecture
(that is, the reproducing kernel Hilbert space, or RKHS).
Such kernels can be defined recursively using certain choices of dot-product kernels at
each layer that depend on the activation function.
For the convolutional case with rectified linear unit (ReLU) activations and arbitrary patches
and linear pooling operations,
we show that the NTK can be expressed through kernel feature maps defined in a tree-structured hierarchy.

We study smoothness and stability properties of the kernel mapping for two-layer networks and CNNs,
which control the variations of functions in the RKHS.
In particular, a useful inductive bias when dealing with natural signals such as images is stability
of the output to deformations of the input, such as translations or small rotations.
A precise notion of stability to deformations was proposed by Mallat~\cite{mallat2012group}, and
was later studied in~\cite{bietti2019group} in the context of CNN architectures, showing the benefits
of different architectural choices such as small patch sizes.
In contrast to the kernels studied in~\cite{bietti2019group},
which for instance cover the limiting kernels that arise from training only the last layer of a ReLU CNN,
we find that the obtained NTK kernel mappings for the ReLU activation
lack a desired Lipschitz property which is needed for stability to deformations in the sense of~\cite{bietti2019group,bruna2013invariant,mallat2012group}.
Instead, we show that a weaker smoothness property similar to Hölder smoothness holds,
and this allows us to show that the kernel mapping is stable to deformations, albeit with a different guarantee.

In order to balance our observations on smoothness, we also consider approximation
properties for the NTK of two-layer ReLU networks,
by characterizing the RKHS using a Mercer decomposition of the kernel in
the basis of spherical harmonics~\cite{bach2017breaking,scholkopf2001learning,smola2001regularization}.
In particular, we study the decay of eigenvalues for this decomposition, which is then related
to the regularity of functions in the space, and provides rates of approximation for Lipschitz functions~\cite{bach2017breaking}.
We find that the full NTK has better approximation properties
compared to other function classes typically defined for ReLU activations~\cite{bach2017breaking,cho2009kernel,daniely2016toward},
which arise for instance when only training the weights in the last layer,
or when considering Gaussian process limits of ReLU networks (\eg,~\cite{garriga2019deep,lee2018deep,matthews2018gaussian,novak2019bayesian}).

\paragraph{Contributions.} Our main contributions can be summarized as follows:
\begin{itemize}[noitemsep,leftmargin=*]
	\item We provide a derivation of the NTK for convolutional networks with generic linear operators for
	patch extraction and pooling, and express the corresponding kernel feature map hierarchically
	using these operators.
	\item We study smoothness properties of the kernel mapping for ReLU networks, showing that it is not Lipschitz but satisfies a weaker Hölder smoothness property. For CNNs, we then provide a guarantee on deformation stability.
	\item We characterize the RKHS of the NTK for two-layer ReLU networks by providing a spectral decomposition of the kernel and studying its spectral decay.
	This leads to improved approximation properties compared to other function classes based on ReLU.
\end{itemize}

\paragraph{Related work.}
Neural tangent kernels were introduced in~\cite{jacot2018neural}, and similar ideas were used to obtain more quantitative
guarantees on the global convergence of gradient descent for over-parameterized
neural networks~\cite{allen2019convergence,arora2019exact,chizat2018note,du2019bgradient,du2019agradient,li2018learning,xie2017diverse,zou2019stochastic}.
The papers~\cite{arora2019exact,du2019bgradient,yang2019scaling} also derive NTKs for convolutional networks,
but focus on simpler architectures.
Kernel methods for deep neural networks were studied for instance in~\cite{cho2009kernel,daniely2016toward,mairal2014convolutional}.
Stability to deformations was originally introduced in the context of the scattering representation~\cite{bruna2013invariant,mallat2012group},
and later extended to neural networks through kernel methods in~\cite{bietti2019group}.
The inductive bias of optimization in neural network learning was considered,
\eg, by~\cite{allen2018learning,arora2019fine,cao2019generalization,neyshabur2015search,soudry2018implicit}.
\cite{bach2017breaking,ghorbani2019linearized,savarese2019infinite,williams2019gradient} study function spaces corresponding to two-layer ReLU networks.
In particular,~\cite{ghorbani2019linearized} also analyzes properties of the NTK, but studies a specific high-dimensional limit for generic activations,
while we focus on ReLU networks, studying the corresponding eigenvalue decays in finite dimension.

\section{Neural Tangent Kernels}
\label{sec:ntk}

In this section, we provide some background on ``lazy training'' and neural tangent kernels (NTKs),
and introduce the kernels that we study in this paper.
In particular, we derive the NTK for generic convolutional architectures on~$\ell^2$ signals.
For simplicity of exposition, we consider scalar-valued functions, noting that the kernels may
be extended to the vector-valued case, as done, \eg, in~\cite{jacot2018neural}.

\subsection{Lazy training and neural tangent kernels}
\label{sub:background}

Multiple recent works studying global convergence of gradient descent in neural networks (\eg,~\cite{allen2019convergence,du2019bgradient,du2019agradient,jacot2018neural,li2018learning,zou2019stochastic})
show that when a network is sufficiently over-parameterized, weights remain close
to initialization during training. The model is then well approximated by its linearization around initialization.
For a neural network $f(x; \theta)$ with parameters~$\theta$ and initialization~$\theta_0$, we then have:\footnote{While we use
gradients in our notations, we note that weak differentiability (\eg, with ReLU activations) is sufficient when studying the limiting NTK~\cite{jacot2018neural}.}
\begin{equation}
\label{eq:tangent_model}
f(x; \theta) \approx f(x; \theta_0) + \langle \theta - \theta_0, \nabla_\theta f(x; \theta_0) \rangle.
\end{equation}
This regime where weights barely move has also been referred to as ``lazy training''~\cite{chizat2018note},
in contrast to other situations such as the ``mean-field'' regime (\eg,~\cite{chizat2018global,mei2018mean,mei2019mean}),
where weights move according to non-linear dynamics.
Yet, with sufficient over-parameterization, the (non-linear) features $x \mapsto \nabla_\theta f(x; \theta_0)$
of the linearized model~\eqref{eq:tangent_model} become expressive enough to be able to
perfectly fit the training data, by approximating a kernel method.

\paragraph{Neural Tangent Kernel (NTK).}
When the width of the network tends to infinity, %
assuming an appropriate initialization on weights,
the features of the linearized model tend to a limiting kernel~$K$, called \emph{neural tangent kernel}~\cite{jacot2018neural}:
\begin{equation}
\label{eq:ntk}
\langle \nabla_\theta f(x; \theta_0), \nabla_\theta f(x', \theta_0) \rangle \to K(x, x').
\end{equation}

In this limit and under some assumptions, one can show that the weights move very slightly and the kernel remains fixed
during training~\cite{jacot2018neural}, and that gradient descent will then lead to the minimum norm
kernel least-squares fit of the training set in the case of the $\ell_2$ loss (see~\cite{jacot2018neural} and~\cite[Section H.7]{mei2019mean}).
Similar interpolating solutions have been found to perform well for generalization, both in practice~\cite{belkin2018understand} and in theory~\cite{bartlett2019benign,liang2018just}.
When the number of neurons is large but finite, one can often show that the kernel only deviates slightly from
the limiting NTK, at initialization and throughout training, thus allowing convergence as long as the initial kernel matrix is non-degenerate~\cite{arora2019exact,chizat2018note,du2019bgradient,du2019agradient}.

\paragraph{NTK for two-layer ReLU networks.}
Consider a two layer network of the form $f(x; \theta) = \sqrt{\frac{2}{m}} \sum_{j=1}^m v_j \sigma(w_j^\top x)$, where $\sigma(u) = (u)_+ = \max(0, u)$ is the ReLU activation, $x \in \R^p$, and $\theta = (w_1^\top, \ldots, w_m^\top, v^\top)$ are parameters with values initialized as~$\mathcal N(0, 1)$.
Practitioners often include the factor~$\sqrt{2/m}$ in the variance of the initialization of~$v_j$,
but we treat it as a scaling factor following~\cite{du2019bgradient,du2019agradient,jacot2018neural},
noting that this leads to the same predictions.
The factor~$2$ is simply a normalization constant specific to the ReLU activation and commonly used by practitioners,
which avoids vanishing or exploding behavior for deep networks.
The corresponding NTK is then given by~\cite{chizat2018note,du2019agradient}:
\begin{align}
K(x, x') &= 2(x^\top x') \E_{w \sim \mathcal N(0, I)}[1\{w^\top x \geq 0\} 1\{w^\top x' \geq 0\}]
			+ 2\E_{w \sim \mathcal N(0, I)}[(w^\top x)_+ (w^\top x')_+] \nonumber \\
		 &= \|x\| \|x'\| \kappa \left( \frac{\langle x, x' \rangle}{\|x\| \|x'\|} \right), \label{eq:two_layer_ntk}
\end{align}
where
\begin{align}
\kappa(u) &:= u \kappa_0(u) + \kappa_1(u) \label{eq:two_layer_ntk_sphere} \\
\kappa_0(u) = \frac{1}{\pi} \left( \pi - \arccos(u) \right),  &\qquad
\kappa_1(u) = \frac{1}{\pi} \left( u \cdot (\pi - \arccos(u)) + \sqrt{1 - u^2} \right) \label{eq:kappa_arccos}.
\end{align}
The expressions for~$\kappa_0$ and~$\kappa_1$ follow from standard calculations for arc-cosine kernels of degree 0 and 1 (see~\cite{cho2009kernel}).
Note that in this two-layer case, the non-linear features obtained for finite neurons correspond to a random features kernel~\cite{rahimi2008random},
which is known to approximate the full kernel relatively well even with a moderate amount of neurons~\cite{bach2017equivalence,rahimi2008random,rudi2017generalization}.
One can also extend the derivation to other activation functions, which may lead to explicit expressions for the kernel in some cases~\cite{daniely2016toward}.

\paragraph{NTK for fully-connected deep ReLU networks.}
We define a fully-connected neural network by $f(x; \theta) = \sqrt{\frac{2}{m_n}} \langle w^{n+1},  \fm^n \rangle$, with
$\fm^1 = \sigma(W^1 x)$, and
\begin{align*}
\fm^k &= \sigma \left( \sqrt{\frac{2}{m_{\kmone}}}W^k \fm^{\kmone} \right), \quad k = 2, \ldots, n,
\end{align*}
where $W^k \in \R^{m_k \times m_{\kmone}}$ and $w^{n+1} \in \R^{m_n}$ are initialized with i.i.d.~$\mathcal N(0, 1)$ entries, and~$\sigma(u) = (u)_+$ is the ReLU activation and is applied element-wise.
Following~\cite{jacot2018neural}, the corresponding NTK is defined recursively by $K(x, x') = K_n(x, x')$ with $K_0(x, x') = \Sigma_0(x, x') = x^\top x'$, and for $k \geq 1$,
\begin{align*}
\Sigma_k(x, x') &= 2\E_{(u, v) \sim \mathcal N(0, B_k)}[\sigma(u) \sigma(v)] \\
K_k(x, x') &= \Sigma_k(x, x') + 2K_{\kmone}(x, x') \E_{(u, v) \sim \mathcal N(0, B_k)}[\sigma'(u) \sigma'(v)],
\end{align*}
where $B_k = \left(\begin{matrix}
	\Sigma_{\kmone}(x, x) & \Sigma_{\kmone}(x, x') \\ \Sigma_{\kmone}(x, x') & \Sigma_{\kmone}(x', x')
\end{matrix} \right)$.
Using a change of variables and definitions of arc-cosine kernels of degrees 0 and 1~\cite{cho2009kernel}, it is easy to show that
\begin{align}
2\E_{(u, v) \sim \mathcal N(0, B_k)}[\sigma(u) \sigma(v)] &= \sqrt{\Sigma_{\kmone}(x, x) \Sigma_{\kmone}(x', x')} \kappa_1 \left( \frac{\Sigma_{\kmone}(x, x')}{\sqrt{\Sigma_{\kmone}(x, x) \Sigma_{\kmone}(x', x')}} \right) \label{eq:arccos1}\\
2\E_{(u, v) \sim \mathcal N(0, B_k)}[\sigma'(u) \sigma'(v)] &= \kappa_0 \left( \frac{\Sigma_{\kmone}(x, x')}{\sqrt{\Sigma_{\kmone}(x, x) \Sigma_{\kmone}(x', x')}} \right), \label{eq:arccos0}
\end{align}
where~$\kappa_0$ and~$\kappa_1$ are defined in~\eqref{eq:kappa_arccos}.

\paragraph{Feature maps construction.}
We now provide a reformulation of the previous kernel in terms of explicit feature maps,
which provides a representation of the data and makes our study of stability in Section~\ref{sec:stability} more convenient.
For a given input Hilbert space~$\Hc$, we denote by $\varphi_{\Hc,1} : \Hc \to \Hc_1$ the kernel
mapping into the RKHS~$\Hc_1$ for the kernel $(z, z') \in \Hc^2 \mapsto \|z\| \|z'\| \kappa_1(\langle z, z' \rangle / \|z\| \|z'\|)$,
and by~$\varphi_{\Hc,0} : \Hc \to \Hc_0$ the kernel mapping into the RKHS~$\Hc_0$ for the
kernel $(z, z') \in \Hc^2 \mapsto \kappa_0(\langle z, z' \rangle / \|z\| \|z'\|)$.
We will abuse notation and hide the input space, simply writing~$\varphi_1$ and $\varphi_0$.

\begin{lemma}[NTK feature map for fully-connected network]
\label{lemma:fcntk}
The NTK for the fully-connected network can be defined as $K(x, x') = \langle \Phi_n(x), \Phi_n(x') \rangle$, with $\Phi_0(x) = \Psi_0(x) = x$ and for $k \geq 1$,
\begin{align*}
\Psi_k(x) &= \varphi_1(\Psi_{\kmone}(x)) \\
\Phi_k(x) &= \left(\begin{matrix} \varphi_0(\Psi_{\kmone}(x)) \otimes \Phi_{\kmone}(x) \\ \varphi_1(\Psi_{\kmone}(x)) \end{matrix} \right),
\end{align*}
where~$\otimes$ is the tensor product.
\end{lemma}

\subsection{Neural tangent kernel for convolutional networks}
\label{sub:conv_ntk}

In this section we study NTKs for convolutional networks (CNNs) on signals, focusing on the ReLU activation.
We consider signals in $\ell^2(\Z^d, \R^{m_0})$, that is, signals~$x[u]$ with~$u \in \Z^d$ denoting the location,
$x[u] \in \R^{m_0}$, and $\sum_{u \in \Z^d} \|x[u]\|^2 < \infty$ (for instance, $d = 2$ and $m_0 = 3$ for RGB images).
The infinite support allows us to avoid dealing with boundary conditions when considering deformations and pooling.
The precise study of~$\ell^2$ membership is deferred to Section~\ref{sec:stability}.

\paragraph{Patch extraction and pooling operators~$P^k$ and~$A^k$.}
Following~\cite{bietti2019group}, we define two linear operators~$P^k$ and~$A^k$ on~$\ell^2(\Z^d)$ for
extracting patches and performing (linear) pooling at layer~$k$, respectively.
For an~$\Hc$-valued signal~$x[u]$, $P^k$ is defined by $P^k x[u] = |S_k|^{-1/2}(x[u + v])_{v \in S_k} \in \Hc^{|S_k|}$,
where~$S_k$ is a finite subset of~$\Z^d$ defining the patch shape (\eg, a 3x3 box).
Pooling is defined as a convolution with a linear filter~$h_k[u]$, \eg, a Gaussian filter at scale~$\sigma_k$ as in~\cite{bietti2019group}, that is, $A^k x[u] = \sum_{v \in \Z^d} h_k[u - v] x[v]$.
In this discrete setting, we can easily include a downsampling operation with factor~$s_k$
by changing the definition of~$A^k$ to $A^k x[u] = \sum_{v \in \Z^d} h_k[s_k u - v] x[v]$ (in particular, if~$h_k$ is a Dirac at 0, we obtain a CNN with ``strided convolutions'').
In fact, our NTK derivation supports general
linear operators~$A^k : \ell^2(\Z^d) \to \ell^2(\Z^d)$ on scalar signals.

For defining the NTK feature map, we also introduce the following non-linear point-wise operator~$M$,
given for two signals~$x, y$, by
\begin{equation}
\label{eq:op_m}
M(x, y)[u] = \left(\begin{matrix}
	\varphi_0(x[u]) \otimes y[u] \\ \varphi_1(x[u])
\end{matrix}\right),
\end{equation}
where~$\varphi_{0/1}$ are kernel mappings of arc-cosine 0/1 kernels,
as defined in Section~\ref{sub:background}.

\paragraph{CNN definition and NTK.}
We consider a network $f(x; \theta) = \sqrt{\frac{2}{m_n}} \langle w^{n+1}, \fm^n \rangle_{\ell^2}$, with
\begin{align*}
\tilde \fm^k[u] &= \begin{cases}
	W^1 P^1 x[u], &\text{ if }k = 1, \\
	\sqrt{\frac{2}{m_{\kmone}}} W^k P^k \fm^{\kmone}[u], &\text{ if }k \in \{2, \dots, n\},
\end{cases} \\
\fm^k[u] &= A^k \sigma(\tilde \fm^k)[u], \quad k = 1, \ldots, n,
\end{align*}
where $W^k \in \R^{m_k \times m_{\kmone}|S_k|}$ and $w^n \in \ell^2(\Z^d, \R^{m_n})$ are
initialized with~$\mathcal N(0, 1)$ entries, and~$\sigma(\tilde x^k)$ denotes the signal
with~$\sigma$ applied element-wise to~$\tilde x^k$.
We are now ready to state our result on the NTK for this model.

\begin{proposition}[NTK feature map for CNN]
\label{prop:cntk}
The NTK for the above CNN, obtained when the number of feature maps $m_1, \ldots, m_n \to \infty$ (sequentially),
is given by $K(x, x') = \langle \Phi(x), \Phi(x') \rangle_{\ell^2(\Z^d)}$, with $\Phi(x)[u] = A^n M(\xi_n, \chi_n)[u]$,
where $\chi_n$ and~$\xi_n$ are defined recursively for a given input~$x$ by $\chi_1[u] = \xi_1[u] = P^1 x[u]$, and for $k \geq 2$,
\begin{align*}
\xi_k[u] &= P^k A^\kmone \varphi_1(\xi_{\kmone})[u] \\
\chi_k[u] &= P^k A^\kmone M(\xi_{\kmone}, \chi_\kmone)[u],
\end{align*}
with the abuse of notation~$\varphi_1(x)[u] = \varphi_1(x[u])$ for a signal~$x$.
\end{proposition}
The proof is given in Appendix~\ref{sub:appx_cntk}, where we also show that in the over-parameterization limit,
the pre-activations $\tilde \fm^k_i[u]$ tend to a Gaussian process with
covariance~$\Sigma^k(x, u; x', u') = \langle \xi_k[u], \xi'_k[u'] \rangle$ (this is related
to recent papers~\cite{garriga2019deep,novak2019bayesian} studying Gaussian process limits of Bayesian convolutional networks).
The proof is by induction and relies on similar arguments to~\cite{jacot2018neural} for fully-connected networks,
in addition to exploiting linearity of the operators $P^k$ and~$A^k$, as well as recursive feature maps for hierarchical kernels.
The recent papers~\cite{arora2019exact,yang2019scaling} also study NTKs for certain convolutional networks; in contrast to these works,
our derivation considers general signals in~$\ell^2(\Z^d)$, supports intermediate pooling or downsampling by changing~$A^k$, and provides a more
intuitive construction through kernel mappings and the operators~$P^k$ and~$A^k$.
Note that the feature maps~$x_k$ are defined independently from the~$y_k$, and in fact correspond to more standard multi-layer
deep kernel machines~\cite{bietti2019group,cho2009kernel,daniely2016toward,mairal2016end} or covariance functions of
certain deep Bayesian networks~\cite{garriga2019deep,lee2018deep,matthews2018gaussian,novak2019bayesian}.
They can also be seen as the feature maps of the limiting kernel that arises when only training weights in the last layer and fixing
other layers at initialization (see, \eg,~\cite{daniely2016toward}).

\section{Two-Layer Networks}
\label{sec:two_layer}

In this section, we study smoothness and approximation properties of the RKHS defined by neural tangent kernels
for two-layer networks.
For ReLU activations, we show that the NTK kernel mapping is not Lipschitz, but satisfies a weaker smoothness property.
In Section~\ref{sub:relu_approx}, we characterize the RKHS for ReLU activations and study its approximation properties and benefits.
Finally, we comment on the use of other activations in Section~\ref{sub:activations}.

\subsection{Smoothness of two-layer ReLU networks} %
\label{sub:two_layer_relu}

Here we study the RKHS~$\Hcal$ of the NTK for two-layer ReLU networks, defined in~\eqref{eq:two_layer_ntk},
focusing on smoothness properties of the kernel mapping, denoted~$\Phi(\cdot)$.
Recall that smoothness of the kernel mapping guarantees smoothness of functions~$f \in \Hcal$,
through the relation
\begin{equation}
\label{eq:cs}
|f(x) - f(y)| \leq \|f\|_\Hcal \|\Phi(x) - \Phi(y)\|_\Hcal.
\end{equation}
We begin by showing that the kernel mapping for the NTK is not Lipschitz.
This is in contrast to the kernel~$\kappa_1$ in~\eqref{eq:kappa_arccos}, obtained by fixing the weights in the first layer and
training only the second layer weights ($\kappa_1$ is 1-Lipschitz by~\cite[Lemma 1]{bietti2019group}).
\begin{proposition}[Non-Lipschitzness]
\label{prop:non_lip}
The kernel mapping~$\Phi(\cdot)$ of the two-layer NTK is not Lipschitz:
\[
\sup_{x, y} \frac{\|\Phi(x) - \Phi(y)\|_{\Hc}}{\|x - y\|} \to +\infty.
\]
This is true even when looking only at points~$x, y$ on the sphere.
It follows that the RKHS~$\Hc$ contains unit-norm functions with arbitrarily large Lipschitz constant.
\end{proposition}

Note that the instability is due to~$\varphi_0$, which comes from gradients w.r.t. first layer weigts.
We now show that a weaker guarantee holds nevertheless, resembling 1/2-Hölder smoothness.
\begin{proposition}[Smoothness for ReLU NTK]
\label{prop:holder_smoothness}
We have the following smoothness properties:
\begin{enumerate}[noitemsep,leftmargin=*,topsep=0pt]
	\item For $x, y$ such that $\|x\| = \|y\| = 1$, the kernel mapping~$\varphi_0$ satisfies
$\|\varphi_0(x) - \varphi_0(y)\| \leq \sqrt{\|x - y\|}$.
\item For general non-zero $x, y$, we have
$\|\varphi_0(x) - \varphi_0(y)\| \leq \sqrt{\frac{1}{\min(\|x\|, \|y\|)} \|x - y\|}$.
\item The kernel mapping~$\Phi$ of the NTK then satisfies
\[
\|\Phi(x) - \Phi(y)\| \leq \sqrt{\min(\|x\|, \|y\|) \|x - y\|} + 2\|x - y\|.
\]
\end{enumerate}
\end{proposition}

We note that while such smoothness properties apply to the functions in the RKHS
of the studied limiting kernels, the neural network functions obtained at finite width
and their linearizations around initialization
are not in the RKHS and thus may not preserve such smoothness properties,
despite preserving good generalization properties, as in random feature models~\cite{bach2017equivalence,rudi2017generalization}.
This discrepancy may be a source of instability to adversarial perturbations.

\subsection{Approximation properties for the two-layer ReLU NTK}
\label{sub:relu_approx}

In the previous section, we found that the NTK~$\kappa$ for two-layer ReLU networks yields weaker smoothness guarantees
compared to the kernel~$\kappa_1$ obtained when the first layer is fixed.
We now show that the NTK has better approximation properties, by studying the RKHS through a spectral decomposition of the
kernel and the decay of the corresponding eigenvalues. This highlights a tradeoff between smoothness and approximation.

The next proposition gives the Mercer decomposition of the NTK $\kappa(\langle x, u \rangle)$ in~\eqref{eq:two_layer_ntk_sphere},
where $x, y$ are in the $p-1$ sphere $\Sbb^{p-1} = \{x \in \R^p : \|x\| = 1\}$.
The decomposition is given in the basis of spherical harmonics, as is common for dot-product kernels~\cite{scholkopf2001learning,smola2001regularization},
and our derivation uses results by Bach~\cite{bach2017breaking} on similar decompositions
of positively homogeneous activations of the form~$\sigma_\alpha(u) = (u)^\alpha_+$.
See Appendix~\ref{sec:appx_approx} for background and proofs.
\begin{proposition}[Mercer decomposition of ReLU NTK]
\label{prop:mercer_relu_ntk}
For any $x, y \in \Sbb^{p-1}$, we have the following decomposition of the NTK~$\kappa$:
\vspace{-0.2cm}
\begin{equation}
\label{eq:mercer_relu_ntk}
\kappa(\langle x, y \rangle) = \sum_{k=0}^\infty \mu_k \sum_{j=1}^{N(p,k)} Y_{k,j}(x) Y_{k,j}(y),
\end{equation}
where $Y_{k,j}, j = 1, \ldots, N(p,k)$ are spherical harmonic polynomials of degree~$k$,
and the non-negative eigenvalues~$\mu_k$ satisfy $\mu_0, \mu_1 > 0$, $\mu_k = 0$ if $k = 2j + 1$ with $j \geq 1$,
and otherwise $\mu_k \sim C(p) k^{-p}$ as~$k \to \infty$, with~$C(p)$ a constant depending only on~$p$.
Then, the RKHS is described by:
\begin{equation}
\Hcal = \left\{ f = \sum_{k\geq 0, \mu_k \ne 0} \sum_{j=1}^{N(p,k)} a_{k,j} Y_{k,j}(\cdot)
	\text{~~~~s.t.~~~} \|f\|_\Hcal^2 := \sum_{k\geq 0, \mu_k \ne 0} \sum_{j=1}^{N(p,k)} \frac{a_{k,j}^2}{\mu_k} < \infty \right\}.
\end{equation}
\end{proposition}
The zero eigenvalues prevent certain functions from belonging to the RKHS, namely those with non-zero Fourier coefficients
on the corresponding basis elements (note that adding a bias may prevent such zero eigenvalues~\cite{basri2019convergence}).
Here, a sufficient condition for all such coefficients to be zero is that the function is even~\cite{bach2017breaking}.
Note that for the arc-cosine 1 kernel~$\kappa_1$, we have a faster decay~$\mu_k = O(k^{-p-2})$,
leading to a ``smaller'' RKHS (see Lemma~\ref{lemma:bach_decomp} in Appendix~\ref{sec:appx_approx} and~\cite{bach2017breaking}).
Moreover, the $k^{-p}$ asymptotic equivalent comes from the term~$u \kappa_0(u)$ in the
definition~\eqref{eq:two_layer_ntk_sphere} of~$\kappa$,
which comes from gradients of first layer weights; the second layer gradients yield~$\kappa_1$,
whose contribution to~$\mu_k$ becomes negligible for large~$k$.
We use an identity also used in the recent paper~\cite{ghorbani2019linearized} which compares similar kernels
in a specific high-dimensional limit for generic activations;
in contrast to~\cite{ghorbani2019linearized}, we focus on ReLUs and study eigenvalue decays in finite dimension.
We note that our decomposition uses a uniform distribution on the sphere,
which allows a precise study of eigenvalues and approximation properties of
the RKHS using spherical harmonics.
When the data distribution is also uniform on the sphere, or absolutely continuous
w.r.t.~the uniform distribution,
our obtained eigenvalues are closely related to those of
integral operators for learning problems,
which can determine, \eg, non-parametric rates of convergence (\eg,~\cite{caponnetto2007optimal,fischer2017sobolev})
as well as degrees-of-freedom quantities for kernel approximation (\eg,~\cite{bach2017equivalence,rudi2017generalization}).
Such quantities often depend on the eigenvalue decay of the integral operator,
which can be obtained from~$\mu_k$ after taking multiplicity into account.
This is also related to the rate of convergence of gradient descent in the lazy training regime, which depends on the minimum eigenvalue
of the empirical kernel matrix in~\cite{chizat2018note,du2019bgradient,du2019agradient}.

We now provide sufficient conditions for a function $f:\Sbb^{p-1} \to \R$ to be in~$\Hcal$, as well as
rates of approximation of Lipschitz functions on the sphere, adapting results of~\cite{bach2017breaking} (specifically Proposition 2 and 3 in~\cite{bach2017breaking}) to our NTK setting.
\begin{corollary}[Sufficient condition for $f \in \Hcal$]
\label{prop:derivatives}
Let~$f : \Sbb^{p-1} \to \R$ be an even function such that all $i$-th order derivatives exist and are bounded by~$\eta$ for $0 \leq i \leq s$, with $s \geq p/2$.
Then $f \in \Hcal$ with $\|f\|_\Hcal \leq C(p) \eta$, where~$C(p)$ is a constant that only depends on~$p$.
\end{corollary}
\begin{corollary}[Approximation of Lipschitz functions]
\label{prop:approx_lip}
Let~$f : \Sbb^{p-1} \to \R$ be an even function such that $f(x) \leq \eta$ and $|f(x) - f(y)| \leq \eta \|x - y\|$,
for all~$x, y \in \Sbb^{p-1}$. There is a function $g \in \Hcal$
with $\|g\|_\Hcal \leq \delta$, where $\delta$ is larger than a constant depending only on~$p$, such that
\begin{equation*}
\sup_{x \in \Sbb^{p-1}} |f(x) - g(x)| \leq C(p) \eta \left(\frac{\delta}{\eta}\right)^{-1/(p/2-1)} \log\left(\frac{\delta}{\eta}\right).
\end{equation*}
\end{corollary}
For both results, there is an improvement over~$\kappa_1$, for which Corollary~\ref{prop:derivatives} requires~$s \geq p/2 + 1$
bounded derivatives, and Corollary~\ref{prop:approx_lip} leads to a weaker rate in~$(\delta/\eta)^{-1/(p/2)}$ (see~\cite[Propositions 2 and 3]{bach2017breaking} with~$\alpha = 1$).
These results show that in the over-parameterized regime of the NTK, training multiple layers leads to better approximation properties
compared to only training the last layer, which corresponds to using~$\kappa_1$ instead of~$\kappa$.
In the different regime of ``convex neural networks'' (\eg,~\cite{bach2017breaking,savarese2019infinite}) where neurons can be selected
with a sparsity-promoting penalty, the approximation rates shown in~\cite{bach2017breaking} for ReLU networks are also weaker than for
the NTK in the worst case (though the regime presents benefits in terms of adaptivity), suggesting that perhaps in some situations
the ``lazy'' regime of the NTK could be preferred over the regime where neurons are selected using sparsity.

\paragraph{Homogeneous case.}
When inputs do not lie on the sphere~$\Sbb^{p-1}$ but in~$\R^p$, the NTK for two-layer ReLU networks takes the form of a
homogeneous dot-product kernel~\eqref{eq:two_layer_ntk}, which defines a different RKHS~$\bar \Hcal$ that we characterize below
in terms of the RKHS~$\Hcal$ of the NTK on the sphere.
\begin{proposition}[RKHS of the homogeneous NTK]
\label{prop:homogeneous}
The RKHS~$\bar \Hcal$ of the kernel $K(x, x') = \|x\| \|x'\| \kappa(\langle x, x'\rangle / \|x\| \|x'\|)$ on~$\R^p$ consists of functions of the form
$f(x) = \|x\| g(x / \|x\|)$ with~$g \in \Hcal$, where~$\Hcal$ is the RKHS on the sphere, and we have~$\|f\|_{\bar \Hcal} = \|g\|_\Hcal$.
\end{proposition}
Note that while such a restriction to homogeneous functions may be limiting, one may easily obtain non-homogeneous functions
by considering an augmented variable~$z = (x^\top, R)^\top$ and defining $f(x) = \|z\| g(z / \|z\|)$, where~$g$ is now
defined on the~$p$-sphere~$\Sbb^p$. When inputs are in a ball of radius~$R$, this reformulation preserves regularity properties (see~\cite[Section 3]{bach2017breaking}).

\subsection{Smoothness with other activations}
\label{sub:activations}
\vspace{-0.2cm}
In this section, we look at smoothness of two-layer networks with different activation functions.
Following the derivation for the ReLU in Section~\ref{sub:background}, the NTK for a general activation~$\sigma$ is given by
\begin{equation*}
K_\sigma(x, x') = \langle x, x' \rangle \E_{w \sim \mathcal N(0, 1)}[\sigma'(\langle w, x \rangle) \sigma'(\langle w, x' \rangle)]
			+ \E_{w \sim \mathcal N(0, 1)}[\sigma(\langle w, x \rangle) \sigma(\langle w, x' \rangle)].
\end{equation*}
We then have the following the following result.
\begin{proposition}[Lipschitzness for smooth activations]
\label{prop:smooth_sigma}
Assume that~$\sigma$ is twice differentiable and that the quantities $\gamma_j := \E_{u \sim \mathcal N(0,1)}[(\sigma^{(j)}(u))^2]$
for $j = 0, 1, 2$ are bounded, with~$\gamma_0 > 0$. Then, for~$x, y$ on the unit sphere, the kernel mapping~$\Phi_\sigma$ of~$K_\sigma$ satisfies
\begin{equation*}
\|\Phi_\sigma(x) - \Phi_\sigma(y)\| \leq \sqrt{(\gamma_0 + \gamma_1) \max \left(1,\frac{2\gamma_1 + \gamma_2}{\gamma_0 + \gamma_1} \right)} \cdot \|x - y\|.
\end{equation*}
\end{proposition}
The proof uses results from~\cite{daniely2016toward} on relationships between activations and the corresponding kernels,
as well as smoothness results for dot-product kernels in~\cite{bietti2019group} (see Appendix~\ref{sub:smooth_sigma_proof}).
If, for instance, we consider the exponential activation~$\sigma(u) = e^{u-2}$, we have $\gamma_j = 1$ for all~$j$ (using results from~\cite{daniely2016toward}),
so that the kernel mapping is Lipschitz with constant~$\sqrt{3}$.
For the soft-plus activation $\sigma(u) = \log(1 + e^u)$, we may evaluate the integrals numerically,
obtaining $(\gamma_0, \gamma_1, \gamma_2) \approx (2.31, 0.74, 0.11)$, so that the kernel mapping is Lipschitz with constant $\approx 1.75$.

\section{Deep Convolutional Networks}
\label{sec:stability}

In this section, we study smoothness and stability properties of the NTK kernel mapping for convolutional networks with ReLU activations.
In order to properly define deformations, we consider continuous signals~$x(u)$ in~$L^2(\R^d)$ instead of~$\ell^2(\Z^d)$
(\ie, we have~$\|x\|^2 := \int \|x(u)\|^2 du < \infty$),
following~\cite{bietti2019group,mallat2012group}.
The goal of deformation stability guarantees is to ensure that the data representation (in this case, the kernel mapping~$\Phi$) does not change
too much when the input signal is slightly deformed, for instance with a small translation or rotation of an image---a useful
inductive bias for natural signals.
For a $C^1$-diffeomorphism~$\tau : \R^d \to \R^d$, denoting $L_\tau x(u) = x(u - \tau(u))$ the action operator of the diffeomorphism,
we will show a guarantee of the form
\begin{equation*}
\|\Phi(L_\tau x) - \Phi(x)\| \leq (\omega(\|\nabla \tau\|_\infty) + C \|\tau\|_\infty) \|x\|,
\end{equation*}
where~$\|\nabla \tau\|_\infty$ is the maximum operator norm of the Jacobian~$\nabla \tau(u)$ over~$\R^d$, $\|\tau\|_\infty = \sup_u |\tau(u)|$,
$\omega$ is an increasing function and~$C$ a positive constant.
The second term controls translation invariance, and~$C$ typically decreases with the scale of the last pooling layer ($\sigma_n$ below),
while the first term controls deformation stability, since~$\|\nabla \tau\|_\infty$ measures the ``size'' of deformations.
The function~$\omega(t)$ is typically a linear function of~$t$ in other settings~\cite{bietti2019group,mallat2012group},
here we will obtain a faster growth of order~$\sqrt{t}$ for small~$t$, due to the weaker smoothness that arises from the arc-cosine 0
kernel mappings.

\vspace{-0.2cm}
\paragraph{Properties of the operators.}
In this continuous setup, $P^k$ is now given for a signal~$x \in L^2$ by $P^k x(u) = \lambda(S_k)^{-1/2} (x(u+v))_{v \in S_k}$,
where~$\lambda$ is the Lebesgue measure.
We then have $\|P^k x\| = \|x\|$, and considering normalized Gaussian pooling filters, we have $\|A^k x\| \leq \|x\|$
by Young's inequality~\cite{bietti2019group}.
The non-linear operator~$M$ is defined point-wise analogously to~\eqref{eq:op_m}, and
satisfies $\|M(x, y)\|^2 = \|x\|^2 + \|y\|^2$.
We thus have that the feature maps in the continuous analog of the NTK construction in Proposition~\ref{prop:cntk} are in~$L^2$
as long as~$x$ is in~$L^2$.
Note that this does not hold for some smooth activations, where $\|M(x,y)(u)\|$ may be a positive constant even when~$x(u) = y(u) = 0$,
leading to unbounded~$L^2$ norm for~$M(x,y)$.
The next lemma studies the smoothness of~$M$,
extending results from Section~\ref{sub:two_layer_relu} to signals in~$L^2$.
\begin{lemma}[Smoothness of operator~$M$]
\label{lemma:m_smooth}
For two signals $x, y \in L^2(\R^d)$, we have
\begin{equation}
\label{eq:m_smooth}
\|M(x, y) - M(x', y')\| \leq \sqrt{\min(\|y\|, \|y'\|) \|x - x'\|} + \|x - x'\| + \|y - y'\|.
\end{equation}
\end{lemma}

\paragraph{Assumptions on architecture.}
Following~\cite{bietti2019group}, we introduce an initial pooling layer~$A^0$, corresponding to an anti-aliasing filter,
which is necessary for stability and is a reasonable assumption given that in practice input signals are discrete, with high frequencies typically filtered by an acquisition device.
Thus, we consider the kernel representation~$\Phi_n(x) := \Phi(A^0 x)$, with~$\Phi$ as in Proposition~\ref{prop:cntk}.
We also assume that patch sizes are controlled by the scale of pooling filters, that is
\vspace{-0.1cm}
\begin{equation}
\label{eq:beta}
\sup_{v \in S_k} |v| \leq \beta \sigma_{\kmone},
\end{equation}
for some constant~$\beta$, where~$\sigma_{\kmone}$ is the scale of the pooling operation~$A^{\kmone}$, which typically increases
exponentially with depth, corresponding to a fixed downsampling factor at
each layer in the discrete case. By a simple induction, we can show the following.
\begin{lemma}[Norm and smoothness of $\Phi_n$]
\label{lemma:phi_norm}
We have $\|\Phi_n(x)\| \leq \sqrt{n + 1} \|x\|$, %
and
\begin{align*}
\|\Phi_n(x) - \Phi_n(x')\| \leq (n+1)\|x - x'\| + O(n^{5/4}) \sqrt{\|x\| \|x - x'\|}.
\end{align*}
\end{lemma}

\paragraph{Deformation stability bound.}
We now present our main guarantee on deformation stability for the NTK kernel mapping (the proof is given in Appendix~\ref{sec:appx_stability}).

\begin{proposition}[Stability of NTK]
\label{prop:stability}
Let $\Phi_n(x) = \Phi(A^0 x)$, and assume $\|\nabla \tau\|_\infty \leq 1/2$. We have the following stability bound:
\vspace{-0.1cm}
\begin{align*}
\|\Phi_n(L_\tau x) - \Phi_n(x) \|
	&\leq \left( C(\beta)^{1/2} C n^{7/4} \|\nabla \tau\|_\infty^{1/2} + C(\beta) C' n^2 \|\nabla \tau\|_\infty + \sqrt{n+1} \frac{C''}{\sigma_n} \|\tau\|_\infty \right) \|x\|,
\end{align*}
where~$C, C', C''$ are constants depending only on~$d$, and~$C(\beta)$ also depends on~$\beta$ defined in~\eqref{eq:beta}.
\end{proposition}
Compared to the bound in~\cite{bietti2019group}, the first term shows weaker stability due to faster growth with~$\|\nabla \tau\|_\infty$,
which comes from~\eqref{eq:m_smooth}.
The dependence on the depth~$n$ is also poorer ($n^2$ instead of~$n$),
however note that in contrast to~\cite{bietti2019group}, the norm and smoothness constants of~$\Phi_n(x)$ in Lemma~\ref{lemma:phi_norm} grow with~$n$ here,
partially explaining this gap.
We also note that choosing small~$\beta$ (\ie, small patches in a discrete setting) is more helpful
to improve stability than a small number of layers~$n$, given that~$C(\beta)$ increases polynomially with~$\beta$, while~$n$ typically decreases logarithmically with~$\beta$ when
one seeks a fixed target level of translation invariance (see~\cite[Section 3.2]{bietti2019group}).

By fixing weights of all layers but the last, we would instead obtain feature maps of the form~$A^n x_n$ (using notation from Proposition~\ref{prop:cntk}),
which satisfy the improved stability guarantee of~\cite{bietti2019group}.
The question of approximation for the deep convolutional case is more involved
and left for future work, but it is reasonable to expect that the RKHS for the NTK
is at least as large as that of the simpler kernel with fixed layers before the last,
given that the latter appears as one of the terms in the NTK.
This again hints at a tradeoff between stability and approximation,
suggesting that one may be able to learn less stable but more discriminative functions
in the NTK regime by training all layers.

\paragraph{Numerical experiments.}
We now study numerically the stability of (exact) kernel mapping representations for
convolutional networks with 2 hidden convolutional layers.
We consider both a convolutional kernel network (CKN,~\cite{bietti2019group})
with arc-cosine kernels of degree 1 on patches (corresponding to the kernel
obtained when only training the last layer and keeping previous layers
fixed) and the corresponding NTK.
Figure~\ref{fig:ntk_imnist_deform} shows the resulting average distances,
when considering collections of digits and deformations thereof.
In particular, we find that for small deformations,
the distance to the original image tends to grow more quickly for the NTK compared to the CKN,
as the theory suggests (a square-root growth rate rather than a linear one).
Note also that the relative distances are generally larger for the NTK than for the CKN,
suggesting the CKN may be more smooth.

\begin{figure}[tb]
	\centering
		\begin{subfigure}[c]{.4\textwidth}
	\includegraphics[width=.8\textwidth]{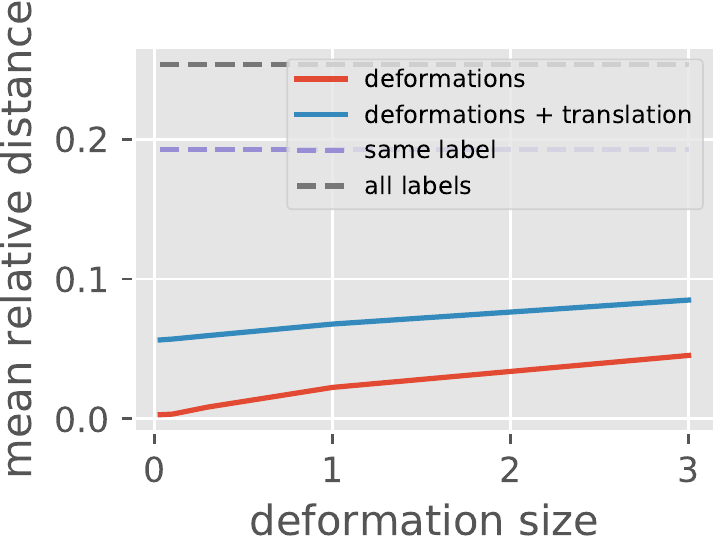}
	\caption{CKN with arc-cosine 1 kernels}
		\end{subfigure}
		\begin{subfigure}[c]{.4\textwidth}
	\includegraphics[width=.8\textwidth]{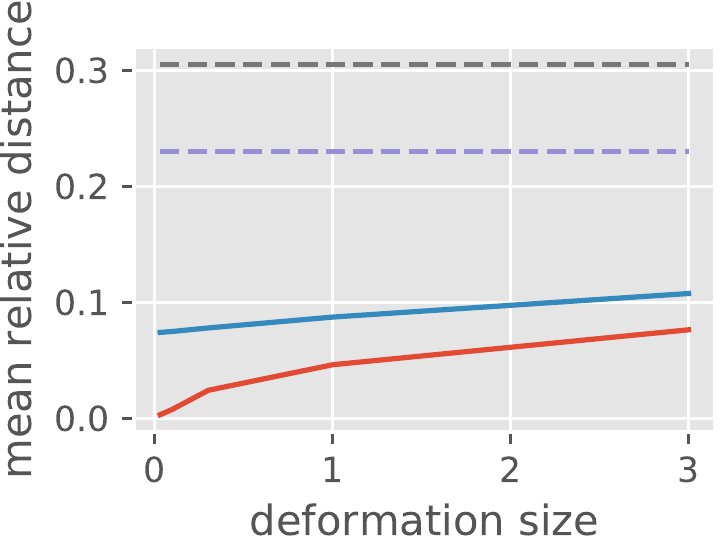}
	\caption{NTK}
	\end{subfigure}
	\caption{Geometry of kernel mapping for CKN and NTK
	convolutional kernels,
	on digit images and their deformations from the Infinite MNIST dataset~\cite{loosli2007training}.
	The curves show average relative distances of a single digit to its deformations, combinations of translations and deformations, digits of the same label, and digits of any label.
	See Appendix~\ref{sec:appx_numerical} for more details on the experimental setup.}
	\label{fig:ntk_imnist_deform}
\end{figure}

\section{Discussion}
\label{sec:discussion}

In this paper, we have studied the inductive bias of the ``lazy training'' regime for over-parameterized neural networks,
by considering the neural tangent kernel of different architectures, and analyzing properties of the corresponding RKHS,
which characterizes the functions that can be learned efficiently in this regime.
We find that the NTK for ReLU networks has better approximation properties compared to other neural network kernels,
but weaker smoothness properties, although these can still guarantee a form of stability to deformations for CNN architectures,
providing an important inductive bias for natural signals.
While these properties may help obtain better performance when large amounts of data are available,
they can also lead to a poorer estimation error when data is scarce, a setting in which smoother kernels or better regularization
strategies may be helpful.

It should be noted that while our study of functions in the RKHS may determine what target functions can
be learned by over-parameterized networks,
the obtained networks with finite neurons do not belong to the same RKHS, and hence may be less stable than such target functions,
at least outside of the training data,
due to approximations both in the linearization~\eqref{eq:tangent_model} and between the finite neuron and limiting kernels.
Additionally, approximation of certain non-smooth functions in this regime may require a very large number of neurons~\cite{yehudai2019power}.
Finally, we note that while this ``lazy'' regime is interesting and could partly explain the success of deep learning methods,
it does not explain, for instance, the common behavior in early layers where neurons move
to select useful features in the data, such as Gabor filters, as pointed out in~\cite{chizat2018note}.
In particular, such a behavior might provide better statistical efficiency by adapting to simple structures in the
data (see, \eg,~\cite{bach2017breaking}), something which is not captured in a kernel regime like the NTK.
It would be interesting to study inductive biases in a regime somewhere in between,
where neurons may move at least in the first few layers.

\subsubsection*{Acknowledgments}

This work was supported by the ERC grant number 714381 (SOLARIS project), the ANR 3IA MIAI@Grenoble Alpes, and by the MSR-Inria joint centre.
The authors thank Francis Bach and Lénaïc Chizat for useful discussions.

\bibliographystyle{abbrv}
\bibliography{full,bibli}

\newpage
\appendix

\section{Proofs of NTK derivations}
\label{sec:appx_ntk}

\subsection{Proof of Lemma~\ref{lemma:fcntk}}
\label{sub:appx_fcntk}

\begin{proof}[Proof of Lemma~\ref{lemma:fcntk}]
By induction, using~\eqref{eq:arccos1} and~\eqref{eq:arccos0} and the corresponding definitions of~$\varphi_1, \varphi_0$,
we can write
\begin{align*}
2\E_{(u, v) \sim \mathcal N(0, B_k)}[\sigma(u) \sigma(v)] &= \langle \varphi_1(\Psi_{\kmone}(x)), \varphi_1(\Psi_\kmone(x')) \rangle \\
2\E_{(u, v) \sim \mathcal N(0, B_k)}[\sigma'(u) \sigma'(v)] &= \langle \varphi_0(\Psi_{\kmone}(x)), \varphi_0(\Psi_\kmone(x')) \rangle.
\end{align*}
The result follows by using the following relation, given three pairs of vectors~$(x,x')$, $(y, y')$ and~$(z, z')$
in arbitrary Hilbert spaces:
\begin{align*}
\langle x, x' \rangle + \langle y, y' \rangle \langle z, z' \rangle = \langle \left(\begin{matrix}
y \otimes z \\ x
\end{matrix}\right), \left(\begin{matrix}
y' \otimes z' \\ x'
\end{matrix}\right) \rangle
\end{align*}
\end{proof}

\subsection{Proof of Proposition~\ref{prop:cntk} (NTK for CNNs)}
\label{sub:appx_cntk}

In this section, we will denote by $\xi_k, \chi_k$ (resp $\xi_k', \chi_k'$) the feature maps
associated to an input $x$ (resp $x'$), as defined in Proposition~\ref{prop:cntk}.
We follow the proofs of Jacot et al.~\cite[Proposition 1 and Theorem 1]{jacot2018neural}.

We begin by proving the following lemma, which characterizes the Gaussian process behavior of the pre-activations~$\tilde \fm^k_i[u]$,
seen as a function of~$x$ and~$u$, in the over-parameterization limit.
\begin{lemma}
\label{lemma:cntk_covariance}
As $m_1, \ldots, m_{\nmone} \to \infty$, the pre-activations~$\tilde \fm^k_i[u]$ for $k = 1, \ldots, n$ tend (in law)
to i.i.d. centered Gaussian processes with covariance
\begin{equation}
\label{eq:cntk_covariance}
\Sigma^k(x, u; x', u') = \langle \xi_k[u], \xi'_k[u'] \rangle.
\end{equation}
\end{lemma}
\begin{proof}
We show this by induction. For $k = 1$, $\tilde \fm^1_i[u]$ is clearly Gaussian, and we have
\begin{align*}
\Sigma^1(x, u; x', u') &= \E[\tilde \fm^1_i[u] \tilde \fm'^1_i[u']] \\
	&= \E[(W^1 P^1 x[u])_i (W^1 P^1 x'[u'])_i].
\end{align*}
Writing $W^k_{ij} \in \R^{|S_k|}$ the vector of weights for the filter associated to the input feature map~$j$ and output feature map~$i$,
we have $(W^1 P^1 x[u])_i = \sum_{j=1}^{m_1} W^{1\top}_{ij} P^1 x_j[u]$. Then we have
\begin{align*}
\Sigma^k(x, u; x', u') &= \sum_{j,j'}\E[W^{1\top}_{ij} P^1 x_j[u] P^1 x'_{j'}[u']^\top W^{1}_{ij'}] \\
	&= \sum_{j,j'} \Tr(\E[W^{1}_{ij'} W^{1\top}_{ij}] P^1 x_j[u] P^1 x'_{j'}[u']^\top) \\
	&= \sum_j \Tr(P^1 x_j[u] P^1 x'_{j}[u']^\top) = \langle P^1 x[u], P^1 x'[u'] \rangle = \langle P^1 \xi_{0}[u], P^1 \xi'_{0}[u'] \rangle,
\end{align*}
by noticing that $\E[W^{1}_{ij'} W^{1\top}_{ij}] = \delta_{j,j'} I_{|S_1|}$.

Now, for $k \geq 2$, we have by similar arguments that conditioned on~$\fm^{\kmone}$, $\tilde \fm^k_i[u]$ is Gaussian,
with covariance
\begin{align*}
\E[\tilde \fm^k_i[u] \tilde \fm'^k_i[u'] | \fm^{\kmone}, \fm'^{\kmone}] &= \frac{2}{m_{\kmone}} \sum_j \langle P^k \fm^{\kmone}_j[u], P^k \fm'^{\kmone}_j[u'] \rangle.
\end{align*}
By the inductive hypothesis, $\tilde \fm^{\kmone}_j[u]$ as a function of~$x$ and~$u$ tend to Gaussian processes in the limit $m_1, \ldots, m_{k-2} \to \infty$.
By the law of large numbers, we have, as $m_{\kmone} \to \infty$,
\begin{align*}
\E[\tilde \fm^k_i[u] \tilde \fm'^k_i[u'] | &\fm^{\kmone}, \fm'^{\kmone}]  \\
	&\to \Sigma^k(x, u; x', u') := 2\E_{f \sim GP(0, \Sigma^{\kmone})}[\langle P^k A^{\kmone} \sigma(f(x))[u], P^k A^{\kmone} \sigma(f(x'))[u'] \rangle].
\end{align*}
Since this covariance is deterministic, the pre-activations~$\tilde \fm^k_i[u]$ are also unconditionally a Gaussian process in the limit,
with covariance~$\Sigma^k$.

Now it remains to show that
\begin{align*}
&2\E_{f \sim GP(0, \Sigma^{\kmone})}[\langle P^k A^{\kmone} \sigma(f(x))[u], P^k A^{\kmone} \sigma(f(x'))[u'] \rangle] \\
&\qquad	= \langle P^k A^{\kmone} \varphi_1(\xi_{\kmone})[u], P^k A^{\kmone} \varphi_1(\xi'_{\kmone})[u'] \rangle.
\end{align*}
Notice that by linearity of~$P^k$ and~$A^{\kmone}$, it suffices to show
\begin{align*}
2\E_{f \sim GP(0, \Sigma^{\kmone})}[\sigma(f(x))[v] \sigma(f(x'))[v']]
	&= \langle \varphi_1(\xi_{\kmone})[v], \varphi_1(\xi'_{\kmone})[v'] \rangle \\
	&= \|\xi_{\kmone}[v]\| \|\xi'_{\kmone}[v']\| \kappa_1 \left( \frac{\langle \xi_{\kmone}[v], \xi_{\kmone}[v'] \rangle}{\|\xi_{\kmone}[v]\| \|\xi'_{\kmone}[v']\|} \right),
\end{align*}
for any~$v, v'$ (the last equality follows from the definition of~$\varphi_1$).
Noting that $(\sigma(f(x))[v], \sigma(f(x'))[v']) = (\sigma(f(x)[v]), \sigma(f(x')[v']))$
when $f \sim GP(0, \Sigma^{\kmone})$ has
Gaussian distribution with zero mean and covariance $\left(\begin{matrix}
	\Sigma^{\kmone}(x, v; x, v) & \Sigma^{\kmone}(x, v; x', v') \\ \Sigma^{\kmone}(x', v'; x', v') & \Sigma^{\kmone}(x', v'; x', v')
\end{matrix}\right)$,
the results follow from~\eqref{eq:arccos1} and~\eqref{eq:cntk_covariance} for~$\kmone$ by the inductive hypothesis.
\end{proof}

We now state and prove a lemma which covers the recursion in the NTK for convolutional
layers (\ie, up to the last fully-connected layer).
\begin{lemma}
\label{lemma:cntk_conv}
As $m_1, \ldots, m_{\nmone} \to \infty$, the gradients of the pre-activations,~$\nabla_\theta \tilde \fm^k_i[u]$, for $k = 1, \ldots, n$ satisfy
\begin{align*}
\langle \nabla_\theta \tilde \fm^k_i[u], \nabla_\theta \tilde \fm'^k_{i'}[u'] \rangle
	&\to \delta_{i,i'}\tilde \Gamma^k_\infty(x, u; x', u')
	= \delta_{i,i'} \langle \chi_k[u], \chi'_k[u'] \rangle.
\end{align*}
\end{lemma}
\begin{proof}
We prove this by induction. For $k = 1$, denoting by $W^1_i$ the $i$th row of~$W^1$, we have
\begin{align*}
\langle \nabla_\theta \tilde \fm^1_i[u], \nabla_\theta \tilde \fm'^1_{i'}[u'] \rangle
	&= \langle \nabla_{W^1} (W^1 P^1 x[u])_i, \nabla_{W^1} (W^1 P^1 x'[u'])_{i'} \rangle \\
	&= \sum_{s} \langle \nabla_{W^1_s} W^1_i P^1 x[u], \nabla_{W^1_s} W^1_{i'} P^1 x'[u'] \rangle \\
	&= \delta_{i,i'} \langle P^1 x[u], P^1 x'[u'] \rangle.
\end{align*}

For $k \geq 2$, assume the result holds up to $\kmone$. We have
\begin{align*}
\langle \nabla_\theta \tilde \fm^k_i[u], \nabla_\theta \tilde \fm'^k_{i'}[u'] \rangle
	&= \langle \nabla_{W^k} \tilde \fm^k_i[u], \nabla_{W^k} \tilde \fm'^k_{i'}[u'] \rangle
		+ \langle \nabla_{W^{1:\kmone}} \tilde \fm^k_i[u], \nabla_{W^{1:\kmone}} \tilde \fm'^k_{i'}[u'] \rangle.
\end{align*}

For the first term, we have, as in the $k = 1$ case,
\begin{align*}
\langle \nabla_{W^k} \tilde \fm^k_i[u], \nabla_{W^k} \tilde \fm'^k_{i'}[u'] \rangle
	&= \frac{2\delta_{i,i'}}{m_{\kmone}} \langle P^k \fm^{\kmone}[u], P^k \fm'^{\kmone}[u'] \rangle \\
	&= \frac{2\delta_{i,i'}}{m_{\kmone}} \sum_j \langle P^k \fm^{\kmone}_j[u], P^k \fm'^{\kmone}_j[u'] \rangle \\
	&= \frac{2\delta_{i,i'}}{m_{\kmone}} \sum_j \langle P^k A^{\kmone} \sigma(\tilde \fm^{\kmone}_j)[u], P^k A^{\kmone} \sigma(\tilde \fm'^{\kmone}_j)[u'] \rangle.
\end{align*}
When $m_1, \ldots, m_{k-2} \to \infty$, $\tilde \fm^{\kmone}_j[u]$ tends to a Gaussian process with covariance $\Sigma^{\kmone}$
by Lemma~\ref{lemma:cntk_covariance}, and when $m_{\kmone} \to \infty$, the quantity above converges to its expectation:
\begin{align*}
\langle \nabla_{W^k} \tilde \fm^k_i[u], \nabla_{W^k} \tilde \fm'^k_{i'}[u'] \rangle
	&\to 2\delta_{i,i'}\E_{f \sim GP(0, \Sigma^{\kmone})} [\langle P^k A^{\kmone} \sigma(f(x))[u], P^k A^{\kmone} \sigma(f(x))[u'] \rangle] \\
	&= \delta_{i,i'} \langle P^k A^\kmone \varphi_1(\xi_{\kmone})[u], P^k A^\kmone \varphi_1(\xi'_{\kmone})[u'] \rangle,
\end{align*}
by using similar arguments to the proof of Lemma~\ref{lemma:cntk_covariance}.

For the second term, identifying all parameters $W^{1:\kmone}$ with a vector $\hat \theta \in \R^q$,
we have by linearity and the chain rule:
\begin{align*}
\sqrt{m_{\kmone}} \nabla_{\hat \theta} \tilde \fm^k_i[u] &= \sum_j \nabla_{\hat \theta} (W^{k\top}_{ij} P^k A^{\kmone} \tilde \fm^{\kmone}_j[u]) \\
	&= \sum_j P^k A^{\kmone} y^{\kmone}_j[u]^\top \cdot W^k_{ij} \in \R^q,
\end{align*}
where $y^{\kmone}_j[u] := \sqrt{2}\sigma'(\tilde \fm^\kmone_j[u]) \nabla_{\hat \theta} \tilde \fm^{\kmone}_j[u] \in \R^q$.
Here we have identified $P^k A^{\kmone} y^{\kmone}_j[u]^\top$ with a matrix in $\R^{q \times |S_k|}$,
where columns are given by $|S_k|^{-1/2} A^{\kmone} y^{\kmone}_j[u + v] \in \R^q$, indexed by $v \in S_k$.
We thus have
\begin{align*}
\langle \nabla_{W^{1:\kmone}} \tilde \fm^k_i[u], \nabla_{W^{1:\kmone}} \tilde \fm'^k_{i'}[u'] \rangle
	&= \frac{1}{m_{\kmone}} \sum_{j, j'} W^{k\top}_{i,j} \underbrace{(P^k A^{\kmone} y^{\kmone}_j[u] \cdot P^k A^{\kmone} y'^{\kmone}_{j'}[u']^\top)}_{=: \Pi^{j,j'} \in \R^{|S_k| \times |S_k|}} W^k_{i',j'}.
\end{align*}
For $v, v' \in S_k$, when $m_1, \ldots, m_{k-2} \to \infty$ and using the inductive hypothesis, we have
\begin{align*}
\Pi^{j,j'}_{v,v'} &= \frac{1}{|S_k|} A^{\kmone} y^{\kmone}_j[u+v]^\top A^{\kmone} y'^{\kmone}_{j'}[u' + v'] \\
	&\to \delta_{j, j'} \bar \Pi^{j}_{v,v'} := \frac{\delta_{j,j'}}{|S_k|} \langle A^{\kmone} \gamma^\kmone_j [u + v], A^{\kmone} \gamma'^\kmone_{j}[u' + v'] \rangle,
\end{align*}
where $\gamma^\kmone_j[u] := \sqrt{2}\sigma'(\tilde \fm^\kmone_j[u]) \chi_{\kmone}[u]$.
Indeed, by linearity it suffices to check that $y^{\kmone}_j[u]^\top y^{\kmone}_{j'}[u'] = 2 \delta_{j,j'} \sigma'(\tilde \fm^\kmone_j[u]) \sigma'(\tilde \fm'^\kmone_{j'}[u']) \langle \chi_{\kmone}[u], \chi'_{\kmone}[u'] \rangle$ for any $j, j', u, u'$,
which is true by the inductive hypothesis.
In this same limit (with $m_{\kmone}$ fixed), we then have
\begin{align*}
\langle \nabla_{W^{1:\kmone}} \tilde \fm^k_i[u], \nabla_{W^{1:\kmone}} \tilde \fm'^k_{i'}[u'] \rangle
	&\to \frac{1}{m_{\kmone}} \sum_{j} W^{k\top}_{i,j} \bar \Pi^j W^k_{i',j}.
\end{align*}
When $m_{\kmone} \to \infty$, by the law of large numbers, this quantity converges to its expectation:
\begin{align*}
\langle \nabla_{W^{1:\kmone}} \tilde \fm^k_i[u], \nabla_{W^{1:\kmone}} \tilde \fm'^k_{i'}[u'] \rangle
	&\to \Tr( \E[W^k_{i',1} W^{k\top}_{i,1}] \Pi^\infty) = \delta_{i,i'} \Tr(\Pi^\infty),
\end{align*}
where $\Pi^\infty$ is given by
\begin{align*}
\Pi^\infty_{v,v'} = \frac{1}{|S_k|}\langle A^{\kmone} \gamma^\kmone_\infty [u + v], A^{\kmone} \gamma'^\kmone_{\infty}[u' + v'] \rangle,
\end{align*}
with $\gamma^{\kmone}_\infty[u] = \varphi_0(\xi_{\kmone}[u]) \otimes \chi_\kmone[u]$.
Indeed, using Lemma~\ref{lemma:cntk_covariance} and linearity of $A^\kmone$, it is enough to check that
\begin{align*}
2\E_{f \sim GP(0, \Sigma^{\kmone})}[\sigma'(f(x)[u]) \sigma'(f(x')[u']) \langle \chi_{\kmone}[u], \chi'_{\kmone}[u'] \rangle ]
	&= \langle \gamma^\kmone_\infty[u], \gamma'^\kmone_\infty[u'] \rangle,
\end{align*}
which holds by definition of~$\varphi_0$ and~$\Sigma^\kmone$.

Finally, notice that
\begin{align*}
\Tr(\Pi^\infty) &= \frac{1}{|S_k|} \sum_{v \in S_k} \langle A^{\kmone} \gamma^\kmone_\infty [u + v], A^{\kmone} \gamma'^\kmone_{\infty}[u' + v] \rangle \\
	&= \langle P^k A^{\kmone} \gamma^\kmone_\infty [u], P^k A^{\kmone} \gamma'^\kmone_{\infty}[u'] \rangle.
\end{align*}

Thus we have
\begin{align*}
\tilde \Gamma_\infty(x, u, x', u') &= \langle P^k A^\kmone \varphi_1(\xi_{\kmone})[u], P^k A^\kmone \varphi_1(\xi'_{\kmone})[u'] \rangle + \langle P^k A^{\kmone} \gamma^\kmone_\infty [u], P^k A^{\kmone} \gamma'^\kmone_{\infty}[u'] \rangle \\
	&= \langle P^k A^\kmone M(\xi_\kmone, \chi_\kmone)[u], P^k A^\kmone M(\xi'_\kmone, \chi'_\kmone)[u'] \rangle \\
	&= \langle \chi_k[u], \chi'_k[u'] \rangle,
\end{align*}
which concludes the proof.

\end{proof}

Armed with the two above lemmas, we can now prove Proposition~\ref{prop:cntk} by studying the gradient of the prediction layer.
\begin{proof}[Proof of Proposition~\ref{prop:cntk}]
We have
\begin{align*}
\langle \nabla_\theta f(x; \theta), \nabla_\theta f(x'; \theta) \rangle
	&= \langle \nabla_{w^{n+1}} f(x; \theta), \nabla_{w^{n+1}} f(x'; \theta) \rangle + \langle \nabla_{W^{1:n}} f(x; \theta), \nabla_{W^{1:n}} f(x'; \theta) \rangle
\end{align*}

The first term writes
\begin{align*}
\langle \nabla_{w^{n+1}} f(x; \theta), \nabla_{w^{n+1}} f(x'; \theta) \rangle
	&= \frac{2}{m_n} \sum_j \langle \fm^n_j, \fm'^n_j \rangle \\
	&= \frac{2}{m_n} \sum_j \sum_u \langle A^n \sigma(\tilde \fm^n_j)[u], A^n \sigma(\tilde \fm'^n_j)[u] \rangle \\
\end{align*}
Using similar arguments as in the above proofs and using Lemma~\ref{lemma:cntk_covariance}, as $m_1, \ldots, m_n \to \infty$, we have
\begin{align*}
\langle \nabla_{w^{n+1}} f(x; \theta), \nabla_{w^{n+1}} f(x'; \theta) \rangle
	&\to \sum_u \langle A^n \varphi_1(\xi_n)[u], A^n \varphi_1(\xi'_n)[u] \rangle = \langle A^n \varphi_1(\xi_n), A^n \varphi_1(\xi'_n) \rangle.
\end{align*}

For the second term, we have
\begin{align*}
\langle \nabla_{W^{1:n}} f(x; \theta), \nabla_{W^{1:n}} f(x'; \theta)
	&= \frac{2}{m_n} \sum_{u,u'} \sum_{j,j'} w^{n+1}_j[u] w^{n+1}_{j'}[u'] \langle \nabla_{W^{1:n}} \fm^n_j[u], \nabla_{W^{1:n}} \fm'^n_{j'}[u'] \rangle.
\end{align*}
We can use similar arguments to the proof of Lemma~\ref{lemma:cntk_conv} to show that when $m_1, \ldots, m_n \to \infty$,
we have
\begin{align*}
\langle \nabla_{W^{1:n}} f(x; \theta), \nabla_{W^{1:n}} f(x'; \theta)
	&\to \sum_{u, u'} \E[w^{n+1}_1[u] w^{n+1}_1[u']] \langle A^n \gamma_n[u], A^n \gamma'_n[u'] \rangle = \langle A^n \gamma_n, A^n \gamma'_n \rangle,
\end{align*}
where $\gamma_n[u] := \varphi_0(\xi_n[u]) \otimes \chi_n[u]$.

The final result follows by combining both terms.
\end{proof}

\section{Proofs for Smoothness and Stability to Deformations}
\label{sec:appx_stability}

\subsection{Proof of Proposition~\ref{prop:non_lip}}
\label{sub:non_lip_proof}

\begin{proof}
Using notations from Section~\ref{sub:background}, we can write
\[
\kappa(u) = u \kappa_0(u) + \kappa_1(u) = \frac{u}{\pi}(\pi - \arccos(u)) + \frac{1}{\pi}(u(\pi - \arccos(u)) + \sqrt{1 - u^2}).
\]
For $\|x\| = \|y\| = 1$, and denoting $u = \langle x, y \rangle$, we have
\begin{align*}
\frac{\|\Phi(x) - \Phi(y)\|^2}{\|x - y\|^2} &= \frac{2\kappa(1) - 2 \kappa(u)}{2 - 2u} \\
	&= \frac{\kappa_0(1) - u\kappa_0(u)}{1 - u} + \frac{\kappa_1(1) - \kappa_1(u)}{1 - u} \\
	&\sim_{u \to 1^-} u \kappa_0'(u) + \kappa_0(u) + \kappa_1'(u) \xrightarrow{u \to 1^-} +\infty,
\end{align*}
where the equivalent follows from l'Hôpital's rule, and we have $\kappa_0'(u) = 1/ \pi \sqrt{1 - u^2} \to +\infty$, while $\kappa_0(1) = \kappa_1(1) = \kappa_1'(1) = 1$.
It follows that the supremum over~$x, y$ is unbounded.

For the second part, fix an arbitrary~$L > 0$.
We can find $x, y$ such that $\|\Phi(x) - \Phi(y)\|_{\mathcal H} > L \|x - y\|$. Take
\begin{align*}
f = \frac{\Phi(x) - \Phi(y)}{\|\Phi(x) - \Phi(y)\|_{\mathcal H}} \in \mathcal H.
\end{align*}
We have $\|f\|_{\mathcal H} = 1$ and $f(x) - f(y) = \|\Phi(x) - \Phi(y)\|_{\mathcal H} > L \|x - y\|$, so that the Lipschitz constant of~$f$ is larger than~$L$.
\end{proof}

\subsection{Proof of Proposition~\ref{prop:holder_smoothness} (smoothness of 2-layer ReLU NTK)}
\label{sub:holder_proof}

\begin{proof}

Denoting $u = \langle x, y \rangle$, we have
\begin{align*}
\frac{\|\varphi_0(x) - \varphi_0(y)\|^2}{\|x - y\|} &= \frac{2 \kappa_0(1) - 2 \kappa_0(u)}{\sqrt{2 - 2u}}.
\end{align*}
As a function of $u \in [-1, 1]$, this quantity decreases from $1$ to $1/2\pi$, and is thus upper bounded by~$1$, proving the first part.

Note that if $u, v$ are on the sphere and $\alpha \geq 1$, then $\|u - \alpha v\| \geq \|u - v\|$.
This yields
\[ \|x - y\| \geq \min(\|x\|, \|y\|) \|\bar x - \bar y\|, \]
where $\bar x, \bar y$ denote the normalized vectors. Then, noting that~$\varphi_0$ is 0-homogeneous, we have
\begin{align*}
\frac{\|\varphi_0(x) - \varphi_0(y)\|^2}{\|x - y\|} &= \frac{\|\varphi_0(\bar x) - \varphi_0(\bar y)\|^2}{\|x - y\|} \\
	&\leq \frac{\|\varphi_0(\bar x) - \varphi_0(\bar y)\|^2}{\min(\|x\|, \|y\|) \|\bar x - \bar y\|} \\
	&\leq \frac{1}{\min(\|x\|, \|y\|)},
\end{align*}
by using the previous result on the sphere, and the result for the second part follows.

For the last part, assume~$x$ has smaller norm than~$y$. We have
\begin{align*}
\|\Phi(x) - \Phi(y)\| &= \left\| \left(\begin{matrix}
	\varphi_0(x) \otimes x \\ \varphi_1(x)
\end{matrix} \right)  -  \left(\begin{matrix}
	\varphi_0(y) \otimes y \\ \varphi_1(y)
\end{matrix} \right)\right\| \\
	&\leq \left\| \left(\begin{matrix}
	\varphi_0(x) \otimes x \\ \varphi_1(x)
\end{matrix} \right) -  \left(\begin{matrix}
	\varphi_0(y) \otimes x \\ \varphi_1(y)
\end{matrix} \right)\right\|
		+ \left\| \left(\begin{matrix}
	\varphi_0(y) \otimes x \\ \varphi_1(y)
\end{matrix} \right)  -  \left(\begin{matrix}
	\varphi_0(y) \otimes y \\ \varphi_1(y)
\end{matrix} \right)\right\| \\
	&= \sqrt{\|x\|^2 \|\varphi_0(x) - \varphi_0(y)\|^2 + \|\varphi_1(x) - \varphi_1(y)\|^2} + \|\varphi_0(y)\|\|x - y\| \\
	&\leq \sqrt{\|x\|\|x - y\| + \|x - y\|^2} + \|x - y\| \leq \sqrt{\|x\| \|x - y\|} + 2\|x - y\|,
\end{align*}
where in the last line we used $\|\varphi_0(y)\| = 1$, $\|\varphi_0(x) - \varphi_0(y)\|^2 \leq \|x - y\| / \|x\|$,
as well as $\|\varphi_1(x) - \varphi_1(y)\| \leq \|x - y\|$, which follows from~\cite[Lemma 1]{bietti2019group}.
We conclude by symmetry.
\end{proof}

\subsection{Proof of Proposition~\ref{prop:smooth_sigma} (smooth activations)}
\label{sub:smooth_sigma_proof}
\begin{proof}
We introduce the following kernels defined on the sphere:
\begin{equation*}
\kappa_j(\langle x, x' \rangle) = \E_{w \sim \mathcal N(0,I)}[\sigma^{(j)}(\langle w, x \rangle)\sigma^{(j)}(\langle w, x' \rangle)].
\end{equation*}
Note that these are indeed dot-product kernels, defined as polynomial expansions in terms of the squared
Hermite expansion coefficients of~$\sigma^{(j)}$,
as shown by Daniely~\cite{daniely2016toward} (called ``dual activations''). In fact, \cite[Lemma 11]{daniely2016toward} also shows that
the mapping from activation to dual activation commutes with differentiation, so that $\kappa_j(u) = \kappa_0^{(j)}(u)$, for $u \in (-1,1)$.
The assumption made in this proposition implies that the $j$-th order derivatives of~$\kappa_0$ as~$u \to 1$ exist,
with $\kappa_0^{(j)}(1) = \kappa_j(1) = \gamma_j < +\infty$.

Then, the NTK on the sphere takes the form $K_\sigma(x, x') = \kappa_\sigma(\langle x, x' \rangle)$,
where $\kappa_\sigma(u) = u \kappa_1(u) + \kappa_0(u)$.
Then, if we consider the kernel $\hat \kappa_\sigma(u) = \frac{\kappa_\sigma(u)}{\kappa_\sigma(1)} = \frac{\kappa_\sigma(u)}{\gamma_0 + \gamma_1}$, we have
\[\hat \kappa_\sigma(1) = 1 \quad \text{ and } \quad \hat \kappa_\sigma'(1) = \frac{\kappa_1(1) + \kappa_1'(1) + \kappa_0'(1)}{\gamma_0 + \gamma_1} = \frac{\gamma_2 + 2 \gamma_1}{\gamma_0 + \gamma_1}.
\]
Applying Lemma 1 of~\cite{bietti2019group} to this kernel, and re-multiplying by~$\gamma_0 + \gamma_1$ yields the final result.
\end{proof}

\subsection{Proof of Lemma~\ref{lemma:m_smooth} (smoothness of operator~$M$ in~$L^2(\R^d)$)}
\label{sub:m_smooth_proof}

\begin{proof}
Using similar arguments as in the proof of Proposition~\ref{prop:holder_smoothness}, we can show that for any~$u \in \R^d$
\begin{align*}
\|M(x, y)(u) - M(x', y')(u)\| &\leq \sqrt{\min(\|y(u)\|, \|y'(u)\|)\|x(u) - x'(u)\|} \\
	&\quad + \|x(u) - x'(u)\| + \|y(u) - y'(u)\|
\end{align*}
Now assume that $\min(\|y\|, \|y'\|) = \|y\|$. By the triangle inequality in~$L^2(\R^d)$, we then have
\begin{align*}
\|M(x, y) - M(x', y')\| &\leq \sqrt{\int \min(\|y(u)\|, \|y'(u)\|)\|x(u) - x'(u)\| du} + \|x - x'\| + \|y - y'\| \\
	&\leq \sqrt{\int \|y(u)\| \|x(u) - x'(u)\| du} + \|x - x'\| + \|y - y'\| \\
	&\leq \sqrt{\|y\| \|x - x'\|} + \|x - x'\| + \|y - y'\|,
\end{align*}
where the last inequality follows from Cauchy-Schwarz.
We obtain the final result by symmetry.
\end{proof}

\subsection{Proof of Proposition~\ref{prop:stability} (stability to deformations)}
\label{sub:stability_proof}

We first recall the following results from~\cite{bietti2019group}.
\begin{lemma}[\cite{bietti2019group}]
\label{lemma:bm}
Assume $\|\nabla \tau\|_\infty \leq 1/2$, and $\sup_{v \in S_k} |v| \leq \beta \sigma_{\kmone}$ for all~$k$.
We have
\begin{align*}
\|[P^k A^{\kmone}, L_\tau]\| &\leq C(\beta) \|\nabla \tau\|_\infty \\
\|L^\tau A^n - A^n\| &\leq \frac{C_2}{\sigma_n}\|\tau\|_\infty,
\end{align*}
where~$C(\beta)$ grows with~$\beta$ as~$\beta^{d+1}$.
\end{lemma}

We are now ready to prove Proposition~\ref{prop:stability}.
\begin{proof}

In order to compare $\Phi_n(L_\tau x)$ and~$\Phi_n(x)$, we introduce intermediate sequences of feature maps,
denoted~$x_k^{(k_0)}$ and~$y_k^{(k_0)}$,
where the deformation operator~$L_\tau$ acts at layer~$k_0$.
In particular, we denote by $x_k^{(0)}$, $y_k^{(0)}$ the feature maps obtained for the input $L_\tau x$,
and if $k_0 \geq 1$, we define $x_k^{(k_0)} = x_k$, $y_k^{(k_0)} = y_k$ for $k \leq k_0$,
\begin{align*}
x_{k_0 + 1}^{(k_0)} &= P^{k_0 + 1} A^{k_0} L_\tau \varphi_1(x_{k_0}) \\
y_{k_0 + 1}^{(k_0)} &= P^{k_0 + 1} A^{k_0} L_\tau M(x_{k_0}, y_{k_0}),
\end{align*}
and for $k \geq k_0 + 2$,
\begin{align*}
x_{k}^{(k_0)} &= P^k A^{\kmone} L_\tau \varphi_1(x_{\kmone}^{k_0}) \\
y_{k}^{(k_0)} &= P^k A^{\kmone} L_\tau M(x_{\kmone}^{k_0},y_{\kmone}^{k_0}).
\end{align*}

Then, we have the following
\begin{align*}
\|\Phi_n(L_\tau x) - \Phi_n(x) \| &= \|A^n M(x_n^{(0)}, y_n^{(0)}) - A^n M(x_n, y_n) \| \\
	&\leq \|A^n M(x_n^{(0)}, y_n^{(0)}) - A^n L_\tau M(x_n, y_n) \| \\
	&\quad+ \|A^n L_\tau M(x_n, y_n) - A^n M(x_n, y_n) \| \\
	&\leq \|A^n M(x_n^{(0)}, y_n^{(0)}) - A^n L_\tau M(x_n, y_n) \| \\
	&\quad+ \|A^n L_\tau - A^n\| \| M(x_n, y_n) \|.
\end{align*}

Using Lemma~\ref{lemma:bm}, we have
\begin{align*}
\|A^n L_\tau - A^n\| &\leq \|[A^n, L_\tau]\| + \|L_\tau A^n - A^n\| \\
	&\leq C(\beta) \|\nabla \tau\|_\infty + \frac{C_2}{\sigma_n}\|\tau\|_\infty.
\end{align*}
Separately, we have $\|M(x_n, y_n)\|^2 = \|x_n\|^2 + \|y_n\|^2 \leq (n+1) \|x\|^2$, so that
\begin{align*}
\|\Phi_n(L_\tau x) - \Phi_n(x) \| &\leq \|A^n M(x_n^{(0)}, y_n^{(0)}) - A^n L_\tau M(x_n, y_n) \| \\
	&\quad+ \sqrt{n+1} \left(C(\beta) \|\nabla \tau\|_\infty + \frac{C_2}{\sigma_n}\|\tau\|_\infty \right) \|x\|.
\end{align*}

We now bound the first term above by induction. For~$n = 1$, we have
\begin{align*}
\|A^1 M(x_1^{(0)}, y_1^{(0)}) - A^1 L_\tau M(x_1, y_1) \| &\leq \|A^1 M(x_1^{(0)}, y_1^{(0)}) - A^1 L_\tau M(x_1, y_1) \| \\
	&\leq \|M(x_1^{(0)}, y_1^{(0)}) - L_\tau M(x_1, y_1) \| \\
	&= \|M(L_\tau x, L_\tau x) - L_\tau M(x, y)\| = 0,
\end{align*}
by noting that~$M$ is a point-wise operator, and thus commutes with~$L_\tau$.
We now assume $n \geq 2$. We have
\begin{align*}
\|A^n M(x_n^{(0)}, y_n^{(0)}) - A^n L_\tau M(x_n, y_n) \|
	&\leq \|M(x_n^{(0)}, y_n^{(0)}) - L_\tau M(x_n, y_n) \| \\
	&= \|M(x_n^{(0)}, y_n^{(0)}) - M(L_\tau x_n, L_\tau y_n) \| \\
	&\leq \sqrt{\|L_\tau y_n\| \|x_n^{(0)} - L_\tau x_n\|} \\
		&\quad+ \|x_n^{(0)} - L_\tau x_n\| + \|y_n^{(0)} - L_\tau y_n\|,
\end{align*}
where we used Lemma~\ref{lemma:m_smooth} and the fact that~$M$ commutes with~$L_\tau$.

Now note that since~$\|\nabla \tau\|_\infty \leq 1/2$, for any signal~$x$ we have
\begin{align}
\|L_\tau x\|^2 &= \int \|x(u - \tau(u))\|^2 du = \int \|x(u)\|^2 |\det(I - \nabla \tau(u))|^{-1}du \nonumber \\
	&\leq \frac{1}{(1 - \|\nabla \tau\|_\infty)^d} \|x\|^2 \leq 2^{d} \|x\|^2. \label{eq:ltnorm}
\end{align}
Thus, we have $\|L_\tau y_n\| \leq 2^{d/2}\|y_n\| \leq \sqrt{2^d n} \|x\|$. Separately,
using the non-expansivity of~$\varphi_1$, we have
\begin{align*}
\|x_n^{(0)} - L_\tau x_n\| &\leq \sum_{k=1}^{\nmone} \|x_n^{(\kmone)} - x_n^{(k)}\| + \|x_n^{(n-1)} - L_\tau x_n\| \\
	&\leq \sum_{k=1}^{n} \|x_k^{(\kmone)} - L_\tau x_k\| \\
	&= \|P^1 A^0 L_\tau x - L_\tau P^1 A^0 x\| + \sum_{k=2}^{n} \|P^k A^{\kmone} L_\tau \varphi_1(x_\kmone) - L_\tau P^k A^\kmone \varphi_1(x_\kmone)\| \\
	&\leq \sum_{k=1}^{n} \|[P^k A^\kmone, L_\tau]\| \|x\| \\
	&\leq C(\beta) n \|\nabla \tau\|_\infty \|x\|,
\end{align*}
by Lemma~\ref{lemma:bm}.
We also have
\begin{align*}
\|y_n^{(0)} - L_\tau y_n\| &= \|P^n A^{\nmone} M(x_\nmone^{(0)}, y_\nmone^{(0)}) - L_\tau P^n A^\nmone M(x_\nmone, y_\nmone)\| \\
	&\leq  \|P^n A^{\nmone} M(x_\nmone^{(0)}, y_\nmone^{(0)}) - P^n A^\nmone L_\tau M(x_\nmone, y_\nmone)\| 
		+ \|[P^n A^\nmone, L_\tau]\| \|M(x_\nmone, y_\nmone)\| \\
	&\leq \| A^{\nmone} M(x_\nmone^{(0)}, y_\nmone^{(0)}) - A^\nmone L_\tau M(x_\nmone, y_\nmone)\| 
		+ C(\beta) \sqrt{n} \|\nabla \tau\|_\infty \|x\|.
\end{align*}
We have thus shown:
\begin{align*}
\|A^n M(x_n^{(0)}, y_n^{(0)}) - A^n L_\tau M(x_n, y_n) \|
 	&\leq \left( 2^{d/4} C(\beta)^{1/2} n^{3/4} \|\nabla \tau\|_\infty^{1/2} + C(\beta) (n + \sqrt{n}) \|\nabla \tau\|_\infty \right) \|x\| \\
 		&\quad + \| A^{\nmone} M(x_\nmone^{(0)}, y_\nmone^{(0)}) - A^\nmone L_\tau M(x_\nmone, y_\nmone)\|.
\end{align*}
Unrolling the recurrence relation yields
\begin{align*}
\|A^n M(x_n^{(0)}, y_n^{(0)}) - A^n L_\tau M(x_n, y_n) \|
	&\leq \sum_{k=2}^n \left( 2^{d/4} C(\beta)^{1/2} k^{3/4} \|\nabla \tau\|_\infty^{1/2} + C(\beta) (k + \sqrt{k}) \|\nabla \tau\|_\infty \right) \|x\| \\
	&\leq \left( C(\beta)^{1/2} C n^{7/4} \|\nabla \tau\|_\infty^{1/2} + C(\beta) C' n^2 \|\nabla \tau\|_\infty \right) \|x\|, \\
\end{align*}
where~$C, C'$ are absolute constants depending only on~$d$.

The final bound becomes
\begin{align*}
\|\Phi_n(L_\tau x) - \Phi_n(x) \|
	&\leq \left( C(\beta)^{1/2} C n^{7/4} \|\nabla \tau\|_\infty^{1/2} + C(\beta) C' n^2 \|\nabla \tau\|_\infty + \sqrt{n+1} \frac{C_2}{\sigma_n} \|\tau\|_\infty \right) \|x\|,
\end{align*}
with a different constant~$C'$.

\end{proof}

\section{Approximation Properties}
\label{sec:appx_approx}

\subsection{Background on spherical harmonics}
In this section, we provide some background on spherical harmonic analysis needed for our study of the RKHS.
See~\cite{costas2014spherical,atkinson2012spherical} for references, as well as~\cite[Appendix D]{bach2017breaking}.
We consider inputs on the $p-1$ sphere~$\mathbb S^{p-1} = \{x \in \R^p, \|x\| = 1\}$.

We denote by $Y_{kj}(x)$, $j = 1, \ldots, N(p, k)$, the spherical harmonics of degree~$k$ on~$\mathbb S^{p-1}$,
where $N(p, k) = \frac{2k + p - 2}{k} \left(\begin{matrix}	k + p - 3 \\ p-2 \end{matrix} \right)$.
They form an orthonormal basis of $L^2(\mathbb S^{p-1}, d\tau)$, where $\tau$ is the uniform measure on the sphere.
The index~$k$ plays the role of an integer frequency, as in Fourier series.
We have the addition formula
\begin{equation}
\label{eq:spherical_addition}
\sum_{j=1}^{N(p, k)} Y_{k,j}(x) Y_{k,j}(y) = N(p, k) P_k(\langle x, y \rangle),
\end{equation}
where~$P_k$ is the $k$-th Legendre polynomial in dimension~$p$ (also known as Gegenbauer polynomials),
given by the Rodrigues formula:
\begin{align*}
P_k(t) = (-1/2)^k \frac{\Gamma(\frac{p-1}{2})}{\Gamma(k + \frac{p-1}{2})} (1 - t^2)^{(3-p)/2}
	\left(\frac{d}{dt}\right)^k (1 - t^2)^{k+(d-3)/2}.
\end{align*}

The polynomials~$P_k$ are orthogonal in~$L^2([-1, 1], d\nu)$ where the measure $d \nu$ is given by
$d \nu(t) = (1 - t^2)^{(p-3)/2}dt$, and we have
\begin{equation}
\label{eq:legendre_norm}
\int_{-1}^1 P_k^2(t) (1 - t^2)^{(p-3)/2}dt = \frac{\omega_{p-1}}{\omega_{p-2}} \frac{1}{N(p,k)},
\end{equation}
where~$\omega_{d-1} = \frac{2 \pi^{d/2}}{\Gamma(d/2)}$ denotes the surface of the sphere~$\mathbb S^{d-1}$ in~$d$ dimensions.
Using the addition formula~\eqref{eq:spherical_addition} and orthogonality of spherical harmonics, we can show
\begin{equation}
\label{eq:legendre_dp}
\int P_j(\langle w, x \rangle) P_k(\langle w, y \rangle) d \tau(w) = \frac{\delta_{jk}}{N(p,k)} P_k(\langle x, y \rangle)
\end{equation}
Further, we have the recurrence relation~\cite[Eq. 4.36]{costas2014spherical}
\begin{equation}
\label{eq:legendre_rec}
t P_k(t) = \frac{k}{2k + p - 2} P_{\kmone}(t) + \frac{k + p - 2}{2k + p - 2} P_{k+1}(t),
\end{equation}
for $k \geq 1$, and for $k = 0$ we simply have $t P_0(t) = P_1(t)$.

The Funk-Hecke formula is helpful for computing Fourier coefficients in the basis of spherical harmonics in terms of
Legendre polynomials: for any~$j = 1, \ldots, N(p, k)$, we have
\begin{equation}
\label{eq:funk_hecke}
\int f(\langle x, y \rangle) Y_{k,j}(y) d \tau(y) = \frac{\omega_{p-2}}{\omega_{p-1}} Y_{k,j}(x) \int_{-1}^1 f(t) P_k(t) (1 - t^2)^{(p-3)/2} dt.
\end{equation}

\subsection{Dot-product kernels and spherical harmonics}

In this section, we provide some background on how spherical harmonics can be used for obtaining descriptions of
the RKHS for dot-product kernels on the $p-1$~sphere~\cite{scholkopf2001learning,smola2001regularization}.
We then recall results from Bach~\cite{bach2017breaking} on such a description for kernels arising from
positively-homogeneous activations (\ie, arc-cosine kernels), and on how differentiability of functions on the sphere
relates to the RKHS.

For a positive-definite kernel $K(x, y) = \kappa(\langle x, y \rangle)$ defined on the $p-1$ sphere $\mathbb S^{p-1}$,
we have the following Mercer decomposition:
\begin{equation}
\label{eq:mercer}
\kappa(\langle x, y \rangle) = \sum_{k=0}^\infty \mu_k \sum_{j=1}^{N(p,k)} Y_{k,j}(x) Y_{k,j}(y) = \sum_{k=0}^\infty \mu_k N(p,k) P_k(\langle x, y \rangle),
\end{equation}
where~$\mu_k$ is obtained by computing Fourier coefficients of $\kappa(\langle x, \cdot \rangle)$ using~\eqref{eq:funk_hecke}:
\begin{equation}
\label{eq:mu}
\mu_k = \frac{\omega_{p-2}}{\omega_{p-1}} \int_{-1}^1 \kappa(t) P_k(t) (1 - t^2)^{(p-3)/2}dt.
\end{equation}
Note that we must have~$\mu_k \geq 0$ for all~$k$, since positive definiteness of~$k$ would be violated if this were not true.
Then, the RKHS~$\Hcal$ consists of functions of the form
\begin{align}
f(x) &= \sum_{k=0}^\infty \sum_{j=1}^{N(p,k)} a_{k,j} Y_{k,j}(x), \label{eq:f_decomp} \\
\text{s.t.}\quad \|f\|_\Hc^2 &= \sum_{k=0}^\infty \sum_{j=1}^{N(p,k)} \frac{a_{k,j}^2}{\mu_k} < \infty. \label{eq:f_norm}
\end{align}
In particular, this requires that $a_{k,j} = 0$ for all~$j$, for any~$k$ such that~$\mu_k = 0$.

\paragraph{Relationship with differentiability.}
Following Bach~\cite[Appendix D.3]{bach2017breaking}, if~$f$ is $s$-times differentiable with derivatives bounded by~$\eta$, then
we have $\|(-\Delta)^{s/2} f\|_{L^2(\Sbb^{p-1})} \leq \eta$, where~$\Delta$ is the Laplace-Beltrami
operator on the sphere~\cite{atkinson2012spherical}.
Given a function~$f$ as in~\eqref{eq:f_decomp},
we can use the fact that~$Y_{k,j}$ is an eigenfunction of $-\Delta$ with eigenvalue $k(k + p - 2)$ (see~\cite{atkinson2012spherical,costas2014spherical}) and write $a_{k,j} = a'_{k,j} / (k(k+p-2))^{s/2}$ for $k \geq 1$, where~$a'_{k,j}$ are the Fourier coefficients of~$(-\Delta)^{s/2}f$ and satisfy $\sum_{k,j} a'^2_{k,j} \leq \eta^2$.
We then have
\begin{lemma}
\label{lemma:derivatives}
Assume~$f$ takes the form~$\eqref{eq:f_decomp}$, with $a_{k,j} = 0$ for all~$j$ when~$\mu_k = 0$,
and that~$f$ is $s$-times differentiable with derivatives bounded by~$\eta$.
Then, if $\max_{k \geq 1, \mu_k \ne 0} 1/(k^{2s} \mu_k) < C$, we have~$f \in \Hcal$ with~$\|f\|_\Hcal^2 \leq C' \eta^2$, with $C' = 1/\mu_0 + C$ if $\mu_0 \ne 0$, or $C' = C$ if $\mu_0 = 0$.
\end{lemma}
Indeed, under the stated conditions, we have
\begin{align*}
\|f\|_\Hcal^2 &= \sum_{k=0}^\infty \sum_{j=1}^{N(p,k)} \frac{a_{k,j}^2}{\mu_k} \leq \frac{a_{0,1}^2}{\mu_0} + \max_{k \geq 1, \mu_k \ne 0} 1/(k^{2s} \mu_k) \sum_{k,j} a'^2_{k,j} \\
	&\leq \frac{1}{\mu_0} \eta^2 + C \|(-\Delta)^{s/2}f\|_{L^2(\Sbb)}^2 \leq C' \eta^2.
\end{align*}

\paragraph{Values of~$\mu_k$ for arc-cosine 0/1 kernels~\cite{bach2017breaking}.}
We recall values of the eigenvalues~$\mu_k$ for arc-cosine kernels of degree 0 and 1, which are obtained in~\cite{bach2017breaking}.
For any non-negative integer $\alpha \geq 0$, Bach~\cite{bach2017breaking} considers the positively homogeneous activation $\sigma_\alpha(u) = (u)_+^\alpha$
and derives the following quantities for $k \geq 0$:
\begin{equation*}
\lambda_{\alpha,k} = \frac{\omega_{p-2}}{\omega_{p-1}} \int_{-1}^1 \sigma_\alpha(t) P_k(t) (1 - t^2)^{(p-3)/2} dt.
\end{equation*}
This can be used to derive the decompositions of the arc-cosine kernels introduced in Section~\ref{sec:ntk},
which are defined using expectations on Gaussian variables, but can be expressed using expectations on the sphere as follows:
\begin{align*}
\kappa_\alpha(\langle x, y \rangle) &= 2\E_{w \sim \mathcal N(0, 1)}[\sigma_\alpha(\langle w, x \rangle)\sigma_\alpha(\langle w, y \rangle)] \\
	&= 2\E_{w \sim \mathcal N(0, 1)}[\|w\|^{2 \alpha} \sigma_\alpha(\langle w / \|w\|, x \rangle)\sigma_\alpha(\langle w / \|w\|, y \rangle)] \\
	&= 2\E_{w \sim \mathcal N(0, 1)}[\|w\|^{2 \alpha}] \int \sigma_\alpha(\langle w, x \rangle)\sigma_\alpha(\langle w, y \rangle) d \tau(w),
\end{align*}
where we used $\alpha$-homogeneity of~$\sigma_\alpha$ and the rotational symmetry of the normal distribution,
which implies that~$w/\|w\|$ is uniformly distributed on the sphere, and independent from~$\|w\|$.

For a fixed~$w$, we can express
\begin{align*}
\sigma_\alpha(\langle w, x \rangle) = \sum_{k=0}^\infty \lambda_{\alpha,k} N(p,k) P_k(\langle w, y\rangle).
\end{align*}
Then, using~\eqref{eq:legendre_dp}, we have
\begin{align*}
\kappa_\alpha(\langle x, y \rangle) = 2\E_{w \sim \mathcal N(0, 1)}[\|w\|^{2 \alpha}] \sum_{k=0}^\infty \lambda_{\alpha,k}^2 N(p,k) P_k(\langle x, y \rangle).
\end{align*}
For $\alpha = 0, 1$, this yields decompositions~\eqref{eq:mercer} of $\kappa_0, \kappa_1$ with $\mu_{0,k} = 2 \lambda_{0,k}^2$
and~$\mu_{1,k} = 2p \lambda_{1,k}^2$.
Bach~\cite{bach2017breaking} then shows the following result on the decomposition of~$\sigma_\alpha$,
which we translate to decompositions of kernel functions~$\kappa_\alpha$.
\begin{lemma}[Decomposition of~$\sigma_\alpha$ and~$\kappa_\alpha$~\cite{bach2017breaking}]
\label{lemma:bach_decomp}
For the activation~$\sigma_\alpha$ on the $p-1$ sphere, we have
\begin{itemize}
	\item $\lambda_{\alpha,k} \ne 0$ if $k \leq \alpha$;
	\item $\lambda_{\alpha,k} = 0$ if $k > \alpha$ if $k = \alpha \mod 2$;
	\item $|\lambda_{\alpha,k}| \sim C_\lambda(p,\alpha) k^{-p/2 - \alpha}$ otherwise, for some constant $C_\lambda(p,\alpha)$ depending on~$p$ and~$\alpha$.
\end{itemize}
For $\alpha \in \{0, 1\}$, the eigenvalues for the corresponding kernel~$\kappa_\alpha$ then satisfy
\begin{itemize}
	\item $\mu_{\alpha,k} > 0$ if $k \leq \alpha$;
	\item $\mu_{\alpha,k} = 0$ if $k > \alpha$ if $k = \alpha \mod 2$;
	\item $\mu_{\alpha,k} \sim C_\mu(p,\alpha) k^{-p - 2\alpha}$ otherwise, with~$C_\mu(p,\alpha) = 2 p^\alpha C_\lambda(p,\alpha)^2$.
\end{itemize}
\end{lemma}
Note that the zero eigenvalues imply that a function~$f$ of the form~\eqref{eq:f_decomp} must have $a_{k,j} = 0$ for $k >\alpha$ and $k = \alpha \mod 2$
in order to be in the RKHS for~$\kappa_\alpha$.
A sufficient condition for this to hold is that $f$ is even (resp. odd) when $\alpha$ is odd (resp. even)~\cite{bach2017breaking}.

\subsection{Decomposition of the NTK for two-layer ReLU networks}
\label{sub:appx_two_layer_relu_ntk}

Recall that the NTK for the two-layer ReLU network with inputs on the sphere is given by
\begin{equation*}
\kappa(u) = u \kappa_0(u) + \kappa_1(u).
\end{equation*}

We now prove Proposition~\ref{prop:mercer_relu_ntk}, which shows that the Mercer decomposition in
spherical harmonics~\eqref{eq:mercer} for this NTK satisfies:
\begin{itemize}
	\item $\mu_0, \mu_1 > 0$
	\item $\mu_k = 0$ if $k = 2j + 1, j \geq 1$
	\item $\mu_k \sim C(p) k^{-p}$ otherwise, with~$C(p) = C_\mu(p,0)$.
\end{itemize}
\begin{proof}[Proof of Proposition~\ref{prop:mercer_relu_ntk}]
Using~\eqref{eq:mu} and the recurrence relation~\eqref{eq:legendre_rec} for Legendre polynomials, as well as $t P_0(t) = P_1(t)$, we have
\begin{align*}
\mu_0 &= \mu_{0,1} + \mu_{1,0} \\
\mu_k &= \frac{k}{2k + p - 2}\mu_{0,\kmone} + \frac{k + p - 2}{2k + p - 2} \mu_{0,k+1} + \mu_{1,k}, \quad \text{ for }k \geq 1.
\end{align*}
By Lemma~\ref{lemma:bach_decomp}, we have the desired properties.
\end{proof}

We now briefly discuss how to adapt the approximation results of Bach~\cite{bach2017breaking} to our setting.
\begin{proof}[Proof sketch of Corollaries~\ref{prop:derivatives} and~\ref{prop:approx_lip}]
In Appendix D of~\cite{bach2017breaking}, Bach defines candidate functions~$g:\Sbb^{p-1}$ from functions~$p \in L^2(\Sbb^{p-1})$
as $g(x) = Tp(x) := \int p(w) \sigma_\alpha(w^\top x) d \tau(w)$, with RKHS norm (denoted~$\gamma_2(g)$
in~\cite{bach2017breaking}) given by the smallest $\|p\|_{L^2}$ for~$p$ such that $g = Tp$.
In our case with the NTK, we may simply consider the operator~$\Sigma^{1/2}$ instead
(the self-adjoint square root of the integral operator~$\Sigma$ of~$\kappa$, using notations from~\cite{bach2017equivalence}; see also~\cite{cucker2002mathematical}),
which simply multiplies each fourier coefficient~$a_{k,j}$ in decomposition~\eqref{eq:f_decomp} by~$\sqrt{\mu_k}$,
and obeys the required properties (in fact,~$T$ and~$\Sigma^{1/2}$ are two different square roots of~$\Sigma$~\cite{bach2017equivalence}).

The proofs can then be adapted directly, by noticing that $\sqrt{\mu_k}$ has the same decay properties
as $\lambda_{\alpha,k}$ with~$\alpha = 0$.
For Corollary~\ref{prop:derivatives}, we also provide the key proof ingredients in our framework in Lemma~\ref{lemma:derivatives} for completeness.
\end{proof}

The proof for the homogeneous RKHS is given below.
\begin{proof}[Proof of Proposition~\ref{prop:homogeneous}]
The kernel~$K$ can be written as
\[
K(x, x') = \langle \|x\| \Phi \left(\frac{x}{\|x\|} \right), \|x'\| \Phi\left(\frac{x'}{\|x'\|}\right) \rangle_\Hcal,
\]
where~$\Phi(\cdot)$ is the kernel mapping of the kernel~$\kappa$ on the sphere.
Then the RKHS~$\bar \Hcal$ can be characterized by the following classical result (see, \eg,~\cite[\S 2.1]{saitoh1997integral} or~\cite[Appendix A]{bach2017breaking}):
\begin{align*}
\bar \Hcal &= \{\underbrace{x \mapsto \langle g, \|x\| \Phi \left(\frac{x}{\|x\|} \right) \rangle_\Hcal}_{=: f_g} : g \in \Hcal\} \\
\|f_g\|_{\bar \Hcal} &= \inf \{\|g'\|_\Hcal : g' \in \Hcal \text{ s.t. } f_g = f_{g'}\}.
\end{align*}
Note that the condition~$f_g = f_{g'}$ implies in particular that~$f_g$ and~$f_{g'}$ are equal on the sphere, and thus that~$g = g'$,
so that the infimum is simply equal to~$\|g\|_\Hcal$. This concludes the proof.

\end{proof}

\section{Details on Numerical Experiments}
\label{sec:appx_numerical}

This section provides more details on the experimental setup used for obtaining Figure~\ref{fig:ntk_imnist_deform},
which closely follows the setup in~\cite[Section 3.4]{bietti2019group}.

\begin{figure}[tb]
		\centering
		\includegraphics[width=0.6\textwidth]{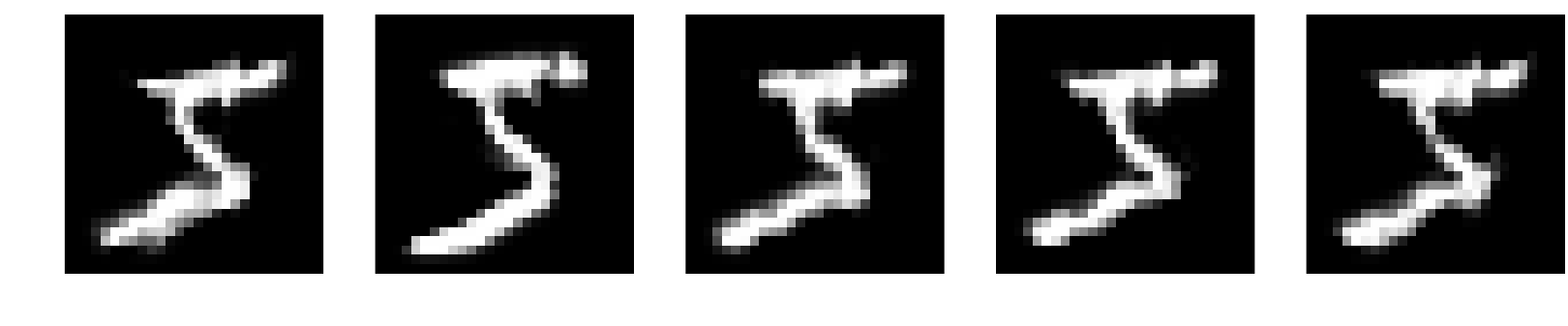}\\
		\includegraphics[width=0.6\textwidth]{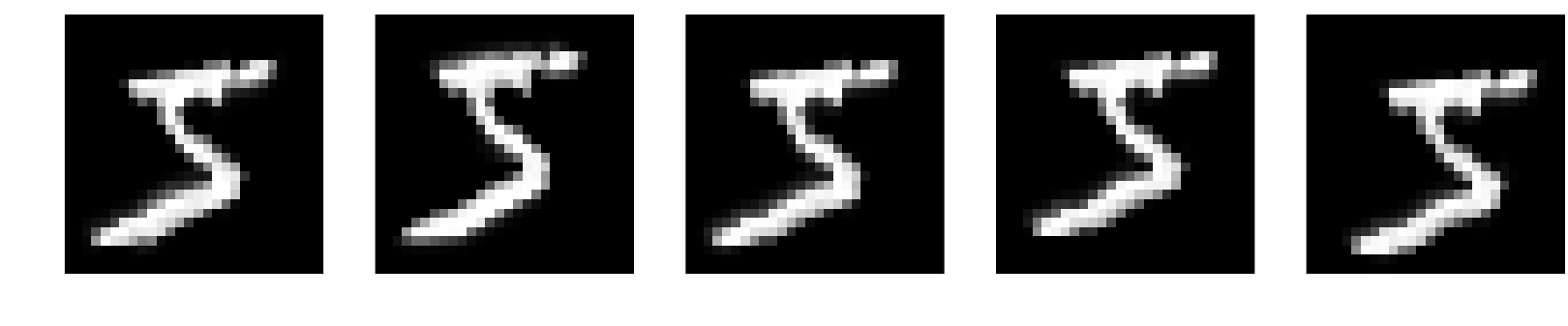}\\
		\includegraphics[width=0.6\textwidth]{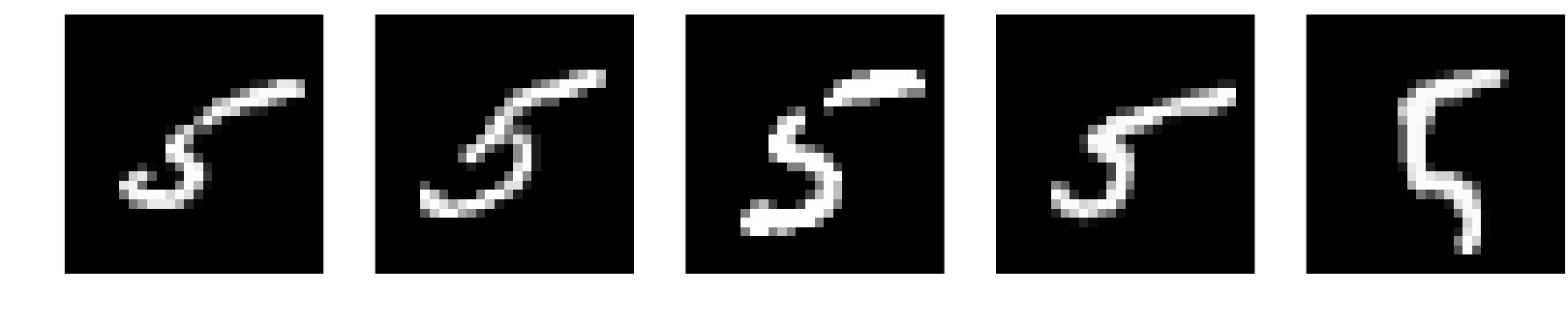}\\
		\includegraphics[width=0.6\textwidth]{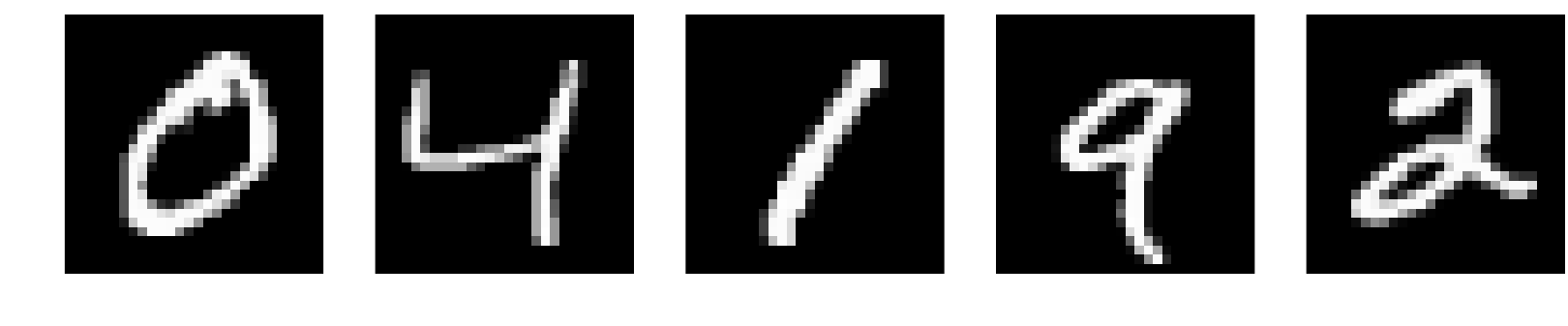}\\
	\caption{MNIST digits with transformations considered in our numerical study of stability.
	Each row gives examples of images from a set of digits that are compared to a reference image of a ``5''.
	From top to bottom:
	deformations with $\alpha = 3$;
	translations and deformations with $\alpha = 1$;
	digits from the training set with the same label ``5'' as the reference digit;
	digits from the training set with any label.
	}
	\label{fig:imnist}
\end{figure}

We consider images of handwritten digits from the Infinite MNIST dataset~\cite{loosli2007training},
which consists of 28x28 grayscale MNIST digits augmented with small translations and deformations.
Translations are chosen at random from one of eight possible directions, while deformations
are generated by considering small smooth deformations $\tau$, and approximating $L_\tau x$
using a tangent vector field $\nabla x$ containing partial derivatives of the signal~$x$ along
the horizontal and vertical image directions.
We introduce a deformation parameter~$\alpha$ to control deformation size. The images are then given by 
\[
L_{\alpha\tau} x(u) = x(u - \alpha\tau(u)) \approx x(u) - \alpha \tau(u) \cdot \nabla x(u).
\]
Figure~\ref{fig:imnist} shows examples of different deformations, with various values of~$\alpha$,
with or without translations, generated from a reference image of the digit ``5''.
In addition, one may consider that a given reference image of a handwritten digit can
be deformed into different images of the same digit, and perhaps even into a different digit (\eg, a ``1'' may be deformed into a ``7'').
Intuitively, the latter transformation corresponds to a ``larger'' deformation than the former,
so that a prediction function that is stable to deformations should be preferable for a classification task.

The architecture underlying the kernels considered in Figure~\ref{fig:ntk_imnist_deform} consists of 2 convolutional layers with patches of size 3x3, followed by ReLU activations, and Gaussian pooling layers adapted to subsampling factors 2 for the first layer and 5 for the second.
Patch extraction is performed with zero padding in order to preserve the size of the previous feature map.
For a subsampling factor~$s$, we apply a normalized Gaussian filter with scale $\sigma = s / \sqrt{2}$ and size $(2s + 1)\times(2s + 1)$, before downsampling.
Our C++ implementation for computing the full kernel given two images
is based on dynamic programming and is available at \url{https://github.com/albietz/ckn_kernel}.

The plots in Figure~\ref{fig:ntk_imnist_deform} show average relative distance in the RKHS between
a reference image and images from various sets of 20 images
(either generated transformations, or images appearing in the training set).
The figures consider deformations of varying size~$\alpha \in \{0.01, 0.03, 0.1, 0.3, 1, 3\}$.
For a given kernel~$K$ and set~$S$ of images, the average relative distance to an image~$x$ is given by
\[
\frac{1}{|S|} \sum_{x' \in S} \frac{\|\Phi(x') - \Phi(x)\|_\Hcal}{\|\Phi(x)\|_\Hcal} = \frac{1}{|S|} \sum_{x' \in S} \frac{\sqrt{K(x, x) + K(x', x') - 2 K(x, x')}}{\sqrt{K(x, x)}},
\]
where~$\Hcal$ denotes the RKHS associated to~$K$ and~$\Phi$ the kernel mapping.
We normalize by~$\|\Phi(x)\|_\Hcal$ in order to reduce sensitivity to the choice of kernel.

\end{document}